\documentclass{article}

\usepackage{PRIMEarxiv}

\usepackage[utf8]{inputenc} % allow utf-8 input
\usepackage[T1]{fontenc}    % use 8-bit T1 fonts
\usepackage{hyperref}       % hyperlinks
\usepackage{url}            % simple URL typesetting
\usepackage{booktabs}       % professional-quality tables
\usepackage{amsfonts}       % blackboard math symbols
\usepackage{nicefrac}       % compact symbols for 1/2, etc.
\usepackage{microtype}      % microtypography
\usepackage{lipsum}
\usepackage{fancyhdr}       % header
\usepackage{graphicx}       % graphics
\graphicspath{{media/}}     % organize your images and other figures under media/ folder
\usepackage{times}
\usepackage{soul}
\usepackage{url}
\usepackage[utf8]{inputenc}
\usepackage[small]{caption}
\usepackage{graphicx}
\usepackage{amsmath}
\usepackage{amsthm}
\usepackage{booktabs}
\usepackage{algorithm}
\usepackage{algorithmic}
\usepackage{wrapfig}
\usepackage{framed}
\usepackage{enumitem}
\usepackage{hyperref}

\usepackage[square,sort,comma,numbers]{natbib}
\usepackage{multirow}
\usepackage{amsmath}  % For mathematical symbols and environments
\usepackage{amsfonts}  % For mathematical fonts (e.g., \mathbb)
\usepackage{graphicx}  % For including graphics (not used in this case)
\usepackage{bm}        % For bold symbols (e.g., \bm for bold math)
\usepackage{mathtools} % For additional math tools
\usepackage{dsfont}
\usepackage{amssymb}
\usepackage{subfigure}
\usepackage{multirow}
\usepackage{transparent}
\usepackage{amsmath,amsthm,amssymb}
\usepackage{sidecap}

\newtheorem{theorem}{Theorem}[section]
\newtheorem{lemma}[theorem]{Lemma}

\newtheorem{definition}[theorem]{Definition}
\newtheorem{proposition}[theorem]{Proposition}
\newtheorem{remark}[theorem]{Remark}
\newtheorem{assumption}{Assumption}

\usepackage{verbatim}
\usepackage{lmodern}
\usepackage[textwidth=3.5cm,textsize=tiny]{todonotes}

\usepackage{color}
\definecolor{shadecolor}{rgb}{0.92,0.92,0.92}

\title{Smart Sampling: Helping from Friendly Neighbors for Decentralized Federated Learning 
}

\author{
Lin Wang$^{1}$, Yang Chen$^{1}$, Yongxin Guo$^{1}$
Xiaoying Tang$^{1}$
 \\
$^1$School of Science and Engineering, The Chinese University of Hong Kong (Shenzhen). 
}

\begin{document}
\maketitle

\begin{abstract}
Federated Learning (FL) is gaining widespread interest for its ability to share knowledge while preserving privacy and reducing communication costs. Unlike Centralized FL, Decentralized FL (DFL) employs a network architecture that eliminates the need for a central server, allowing direct communication among clients and leading to significant communication resource savings. However, due to data heterogeneity, not all neighboring nodes contribute to enhancing the local client's model performance.
In this work, we introduce \textbf{\emph{AFIND+}}, a simple yet efficient algorithm for sampling and aggregating neighbors in DFL, with the aim of leveraging collaboration to improve clients' model performance.
AFIND+ identifies helpful neighbors, adaptively adjusts the number of selected neighbors, and strategically aggregates the sampled neighbors' models based on their contributions. Numerical results on real-world datasets with diverse data partitions demonstrate that AFIND+ outperforms other sampling algorithms in DFL and is compatible with most existing DFL optimization algorithms.
\end{abstract}

% keywords can be removed
% \keywords{First keyword \and Second keyword \and More}

\section{Introduction}
Federated Learning (FL), a collaborative machine learning paradigm, has garnered attention for its capacity to train models on decentralized devices without sharing raw data~\citep{mcmahan2017communication}.  This method addresses privacy and communication concerns, proving valuable in fields like healthcare~\citep{kaissis2020secure}, energy~\citep{saputra2019energy}, and manufacturing~\citep{qu2020blockchained}.
In Centralized FL (CFL), a central server selects clients for training, where effective client sampling is key for accelerating convergence and reducing communication overhead~\citep{cho2020client,wang2022client,shen2022fast}.  However, due to data heterogeneity among clients, it's challenging to attain satisfactory performance for all using a single average central model~\citep{wang2021field,mendieta2022local}.

In contrast, decentralized federated learning (DFL), illustrated in Figure~\ref{CFL vs DFL}, utilizes a peer-to-peer structure where clients exchange model parameters directly, bypassing a central server~\citep{lalitha2019peer}. Each client develops a local, personalized model, addressing the model shift caused by data heterogeneity in Centralized FL (CFL)~\citep{sadiev2022decentralized,jeong2023personalized}. This shift to a decentralized architecture enhances privacy and reduces reliance on central infrastructure \citep{kairouz2019advances}. In particular, this decentralization does not need a global model explicitly or implicitly through the central server, thus significantly reducing the communication burden on the server side~\citep{li2021fedmask}.

Despite its advantages, DFL encounters challenges due to network heterogeneity, affecting collaborative efficiency among clients~\citep{roy2019braintorrent}. Similar to CFL, selecting neighboring nodes\footnote{The terms client, node, and neighbor are used interchangeably in this paper.} for collaboration is crucial in DFL as different collaborations resulting different model performance~\citep{beltran2023decentralized,qu2021decentralized}. 
We break down the neighbor node cooperation problem in DFL into three progressive challenges: identifying the right neighbors for collaboration, adaptively setting the number of participating clients, and valuing the importance of each selected client for aggregation.

% The main challenges in DFL neighbor sampling include identifying the right neighbors for collaboration, enabling adaptive participation beyond a fixed number of clients, and effectively aggregating selected neighbors for collaboration.

\textbf{Challenge 1: Identify the right neighbors for collaboration.} Data heterogeneity across clients results in varied node distributions. Inappropriate neighbor selection can lead to performance degradation compared to solo training, while effective selection significantly boosts performance, as depicted in Fig~\ref{illustration of cooperation}. 
The underlying principle is the similarity in data distribution between clients, such as between Client 1 and Client 2, as opposed to the divergent distribution of Client 3. This highlights the importance of similarity in client sampling. To understand how similarity aids in finding suitable collaborators theoretically, we explore the concept of a coreset for DFL client sampling, discussed in more detail in Section~\ref{sec identify neighbors}.
% The nature behind this is that Client 1 and Client 2 has similar data distribution, while Client 3' distribution diverges far away from Client 2. This motivates us to look at the similarity in client sampling. To theoretical understand how the similarity helping finding right collaborators, we derive the coreset for DFL client sampling, we discuss this in more detail in Section~\ref{sec identify neighbors}.

\textbf{Challenge 2: Adaptive neighbor sampling.}
 % In addition to data heterogeneity, the number of sampled clients is crucial. Fixed-number sampling methods may perform poorly on new tasks without fine-tuning the number, but fine-tuning is resource-intensive and not always effective~\citep{sui2022find}. For instance, in Figure~\ref{illustration of cooperation}, setting the neighbor sampling number to 2 (involving Client 1, Client 2, and Client 3) worsens performance compared to just collaborating between Client 1 and Client 2.
 % To overcome the limitations of fixed-number sampling, we propose a \emph{participation threshold} to evaluate the overall similarity between a client and its neighbors. This threshold adjusts based on the similarity of neighbor distributions to the client’s. A higher threshold is set when many neighbors have diverse distributions, reducing the number of sampled clients to avoid including those who negatively contribute. Conversely, a lower threshold is applied when most neighbors have similar distributions to the client.
 In addition to data heterogeneity, the number of sampled clients is crucial. Fixed-number sampling methods may perform poorly on new tasks without fine-tuning, which is resource-intensive and often ineffective~\citep{sui2022find}. For example, in Figure~\ref{illustration of cooperation}, setting the neighbor sampling number to 2 (involving Client 1, Client 2, and Client 3) worsens performance compared to collaboration between Client 1 and Client 2 alone.
To address the limitations of fixed-number sampling, we propose a \emph{participation threshold} to evaluate the overall similarity between a client and its neighbors. This threshold adjusts based on neighbor distribution similarity. A higher threshold is set when neighbors have diverse distributions, reducing the number of sampled clients to avoid those who negatively contribute. Conversely, a lower threshold is applied when most neighbors have similar distributions to the client.

\textbf{Challenge 3: Aggregate the sampled clients.}
Current DFL client sampling algorithms, after selecting a subset of neighbors for collaboration, often overlook the aggregation step, typically averaging the models received from neighbors~\cite{onoszko2021decentralized}. However, effective aggregation is crucial in FL to maximize the efficiency of sampling~\cite{fraboni2021impact,wang2022client}.
The varying data distributions among sampled clients lead to differing contributions to collaboration. As shown in the experimental results of Figure~\ref{Fig of convergence performance}, contribution-aware aggregation, which considers the individual impact of each client, is more effective than simple averaging, leading to faster convergence.

\begin{figure}
\centering
    \includegraphics[width=0.7\textwidth]{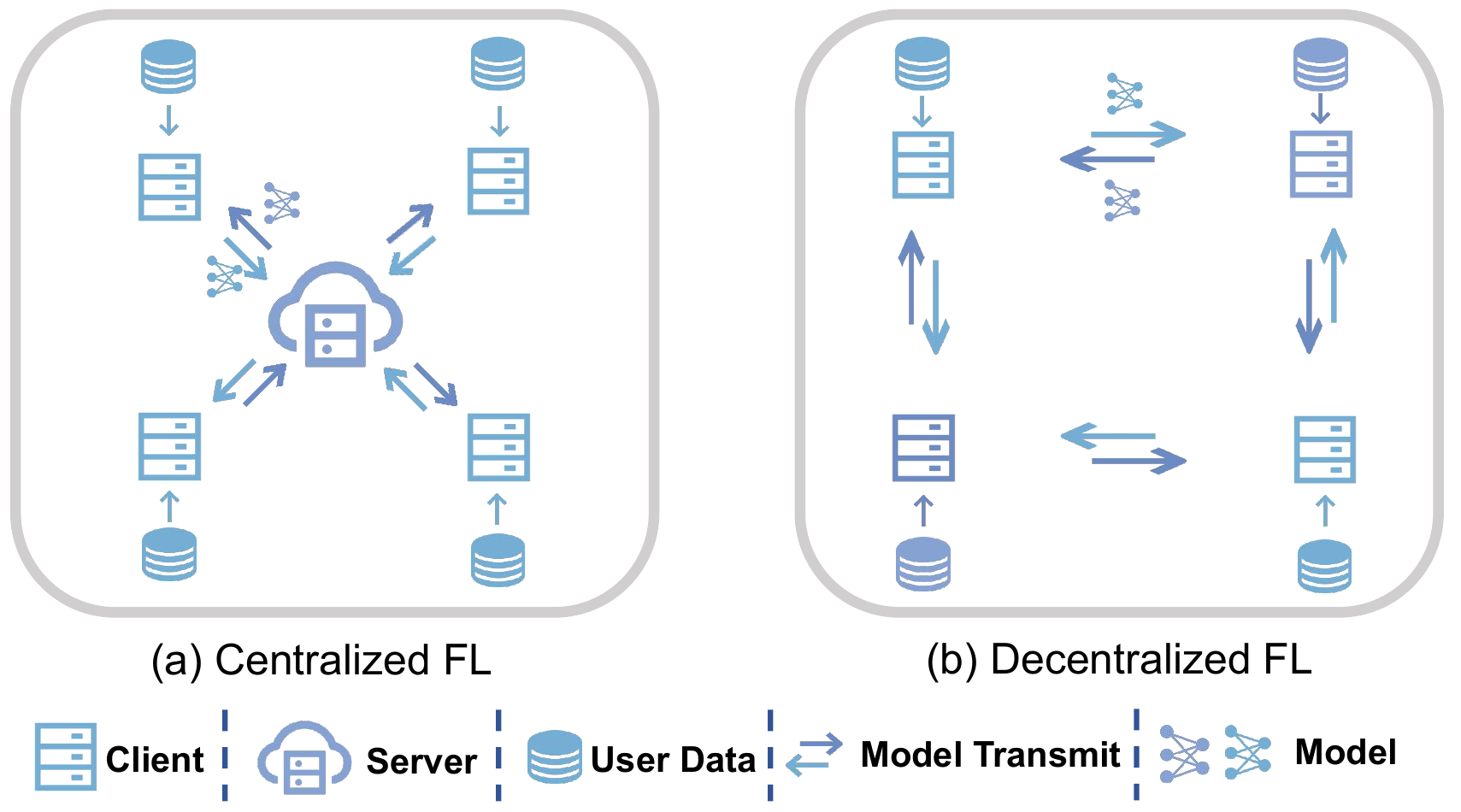}
    \caption{\textbf{Illustration of Centralized FL and Decentralized FL.}
    In centralized FL, communication takes place between the server and the clients, whereas in decentralized FL, it occurs directly between clients without the need for a central server.
    % In centralized FL, communication is between the server and the users, while in decentralized FL, it occurs between the nodes without server.
    }
    \label{CFL vs DFL}
\end{figure}

To this end, we propose \textbf{AFIND+}, an \textbf{A}daptive \textbf{F}r\textbf{I}endly \textbf{N}eighbor \textbf{D}iscory algorithm,  designed to enhance collaboration and performance of DFL.
% In particular, to identify the right neighbors with similar distribution to client, we use the client feature as proxy to measure the similarity between the client and neighbors; to set a adaptive threshold that controls the data quality of sampled neighbors and thus adjust the number of collaborators, we quantify the confidence level of the similarity between client and overall neighbors as a basis for setting thresholds; to improve the training efficient-ness,  we aggregated the sampled neighbors by the Boltzmann distribution of loss.
% AFIND+ employs a three-pronged approach: First, to select the right neighbors with similar data distributions, it uses client features as a proxy for measuring similarity between clients and neighbors. Second, for adaptive threshold setting, which controls the data quality of sampled neighbors and adjusts the number of collaborators, it quantifies the confidence level in the similarity between a client and the overall neighbors. This serves as the basis for threshold determination. Third, to increase training efficiency, it aggregates the contributions of the sampled neighbors using the Boltzmann distribution of loss.
Specifically, based on theoretical and empirical observations, AFIND+ uses the output of a featurizer as a proxy for clients' model similarity to identify neighbors with similar data distributions. To adaptively adjust the number of collaborators, AFIND+ sets a threshold to terminate the greedy selection process using confidence levels in client and neighbor similarity. Among the selected neighbors, AFIND+ assesses their importance to the client and reweights them for aggregation based on their contributions.

% AFIND+ employs client features as proxies to gauge the similarity between a client and its neighbors, aiding in the identification of appropriate collaborators. 
% To adaptively adjust the number of collaborators, AFIND+ establishes an adaptive threshold. This threshold is based on a quantified confidence level, reflecting the similarity between the client and the overall neighbors. To enhance training efficiency, AFIND+ aggregates the sampled neighbors by applying the Boltzmann distribution of loss, thereby incorporating awareness of each neighbor's contribution.

\textbf{The contribution and novelty of AFIND+ can be summarized as follows:}
\begin{itemize}[leftmargin=12pt,nosep]
\setlength{\itemsep}{3.pt}
\item 
% To the best of our knowledge, this is the first work to provide both theoretical understanding and empirical observations that in DFL,  helpful collaborators are those with similar data distributions, and thus propose a greedy selection strategy.
To the best of our knowledge, this work is the first to provide both theoretical insights and empirical evidence that, in DFL, helpful collaborators are those possessing similar data distributions. Based on this insight, we propose a greedy selection strategy to identify such collaborators.
    \item By proposing the use of the confidence level between clients and neighbors, it is the first to enable a dynamic number of collaboration neighbors, which helps set a termination for the greedy selection.
    \item It strategically aggregates the sampled neighbors based on their contributions, instead of merely sampling and leaving them alone, differing from existing methods.
    \item
Theoretically, we provide convergence guarantees for our algorithm in a general nonconvex setting, achieving a convergence rate of $\mathcal{O}\left(\frac{1}{\sqrt{T}} + \frac{1}{T}\right)$. Empirically, we conduct extensive experiments on realistic data tasks using diverse data partition methods, evaluating the efficacy of our algorithm compared with the SOTA sampling baselines in DFL, achieving a maximum improvement of 5\% on CIFAR-10 and 4\% on CIFAR-100.
\end{itemize}

\begin{figure}
\centering
    \includegraphics[width=0.7\textwidth]{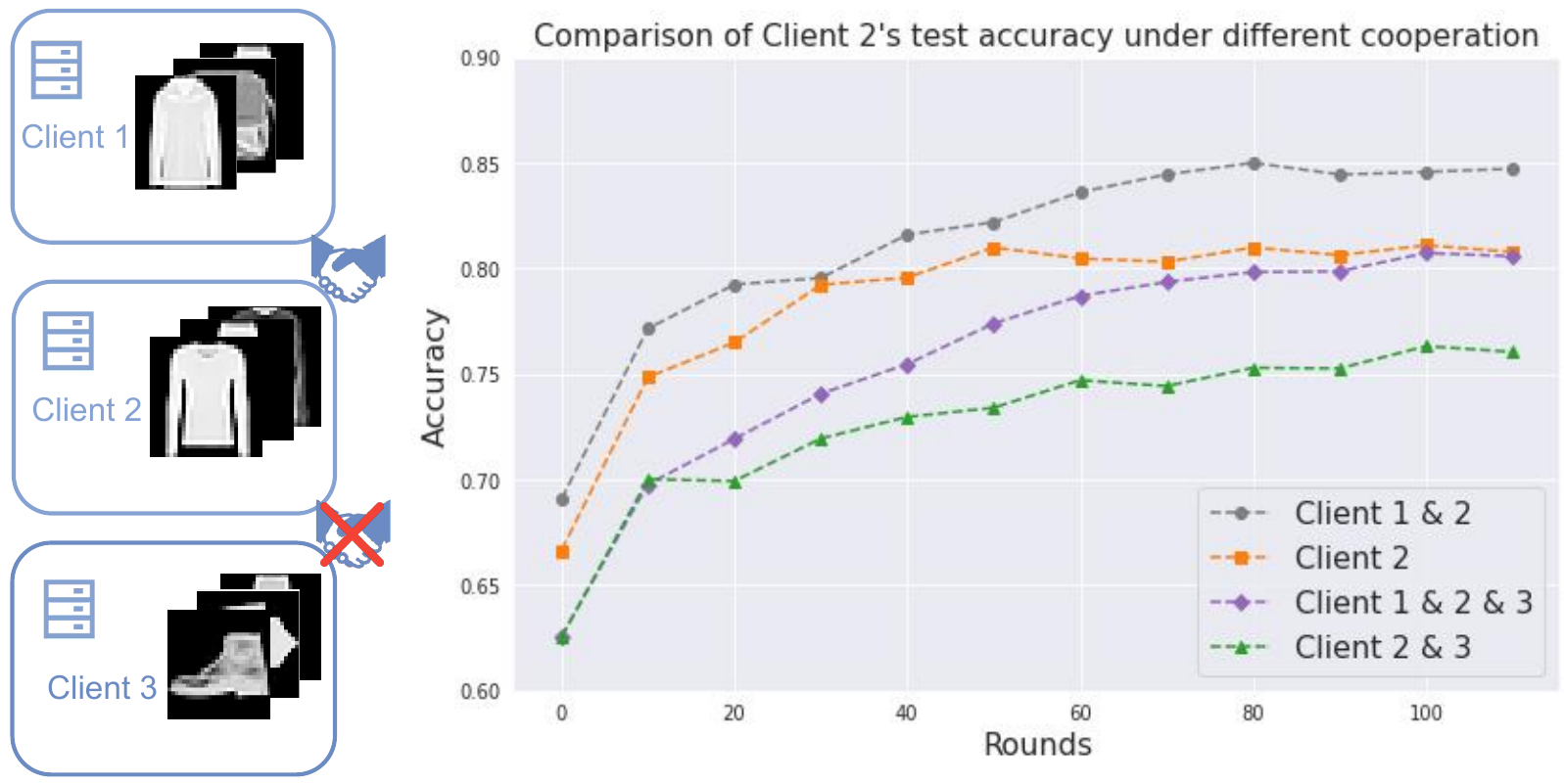}
    % \vspace{-1em}
    \caption{\textbf{Toy examples show the importance of appropriate cooperation.} 
Clients 1 and 2 have FashionMNIST datasets with labels \{4, 5, 6, 7\}, while Client 3's dataset has labels \{0, 1, 2, 3\}. Using Client 2 as a baseline, accuracy comparisons indicate that collaboration between clients with similar data (Clients 1 and 2) is beneficial, whereas cooperation between clients with diverse data (Clients 2 and 3) is detrimental.
}
    \label{illustration of cooperation}
\end{figure}

\section{Related Works}

DFL is a peer-to-peer communication learning paradigm, with clients connecting solely to their neighbors~\citep{liu2022federated, kang2022blockchain, li2022effectiveness}. Neighbor selection has been shown to be an important challenge for DFL~\citep{hegedHus2019gossip, onoszko2021decentralized, sui2022find}. In this work, we address this challenge for DFL with a novel collaboration strategy AFIND+. A more comprehensive discussion of the related work can be found in Appendix~\ref{sec app related works}.

\section{AFIND+: Adaptive Collaboration for DFL}

% In this section, we define the problem setup for DFL and necessary notations first. After that, we present our algorithm AFIND+, including right neighbors finding via coreset (Sec~\ref{sec identify neighbors}), adaptive participation threshold setting via confidence level (Sec~\ref{sec adaptive threshold}) and reweight aggregation based on contribution.
% (Sec~\ref{sec aggregation}).

In this section, we first define the problem setup and necessary notations for DFL. Then, we introduce AFIND+ (Algorithm~\ref{alg:algorithm}), which includes finding the right neighbors via coreset (Sec~\ref{sec identify neighbors}), setting an adaptive participation threshold based on confidence levels (Sec~\ref{sec adaptive threshold}), and contribution-based reweight aggregation (Sec~\ref{sec aggregation}).

\subsection{Preliminary}

% In this subsection, we define used notions and the problem setup for DFL with partial personalized models.

\textbf{Decentralized Federated Learning (DFL).}
In a standard DFL scenario with $m$ clients, each client $i \in [m]$ possesses a local dataset $D_i$ containing $N_i$ data examples. The parameters of the model are denoted by $\mathbf{x}\in \mathbb{R}^d$, and $F_i(\mathbf{x};\xi_i)$ represents the local objective function for client $i$, corresponding to the training samples $\xi_i$. The loss function for client $i$ is given as $F_i(\mathbf{x};\xi_i)$. The typical goal in DFL involves solving the following finite-sum stochastic optimization problem:
% We consider a typical setting of DFL with $m$ clients, where each client $i \in [m]$ has a local dataset $D_i $, with $N_i$ number of examples for client $i$.
% Let $\mathbf{x}\in \mathbb{R}^d$ represent the parameters of a machine learning model and $F_i(\mathbf{x};\xi_i)$ is the local objective function associated with the training data samples $\xi_i$. Then the loss function associated with client $i$ is $F_i(\mathbf{x};\xi)$. Then, a common objective of DFL is the following finite-sum stochastic nonconvex minimization problem:
\begin{align}
    \label{objective of DFL}
    \min_{\mathbf{x}} f(\mathbf{x}) = \frac{1}{m}\sum_{i=1}^m   F_i(\mathbf{x};\xi_i) \, .
\end{align}

In a decentralized network topology, client communication is represented by an undirected graph $\mathcal{G} = (\mathcal{N}, \mathcal{V}, \mathcal{W})$. Here, $\mathcal{N} =[1,\cdots,m]$ denotes the set of clients, $\mathcal{V} \in \mathcal{N} \times \mathcal{N}$ signifies the set of communication channels connecting two distinct clients, and the connection matrix $\mathcal{W} = [w_{i,j}]\in [0,1]^{m\times m}$ indicates the presence of communication links between any two clients~\citep{sun2022decentralized}. We define $S_i$ as the set of nodes that can communicate with client $i$, where $w_{i,j} = 1$ for all $j \in S_i$.

\textbf{DFL with personalized model.} When considering a personalized model approach for DFL, the objective can be formulated as follows:
\begin{align}
\label{personalized formulation}
    \min_{\mathbf{w},\mathbf{\beta}}\left\{ f(\mathbf{w},\mathbf{\beta}) \coloneqq \frac{1}{m}\sum_{i=1}^m  F_i(\mathbf{w},\mathbf{\beta}_i)  \right\} \, ,
\end{align}
where $\mathbf{w}\in \mathbb{R}^{d_0}$ represents the consensus model, averaged from all shared parameters $\mathbf{w}_i$, i.e., $\mathbf{w}=\frac{1}{m}\sum_{i=1}^m \mathbf{w}_i$, $\mathbf{\beta} = (\mathbf{\beta}_1, \cdots, \mathbf{\beta}_m)$ with $\mathbf{\beta}_i \in \mathbb{R}^{d_i}$, $\forall i \in [m]$ corresponds to the local parameters, and $F_i(\mathbf{w},\mathbf{\beta}_i) = \mathbb{E}_{\xi_i \sim D_i}[F_i(\mathbf{w},\mathbf{\beta}_i;\xi_i)]$. 
Stochastic gradients with respect to $\mathbf{w}_i$ and $\mathbf{\beta}_i$ are denoted as $\nabla_{\mathbf{w}}$ and $\nabla_{\mathbf{\beta}}$, respectively.

In DFL, the shared parameters $\mathbf{w}_i$ of each client $i$ are transmitted to their neighbors, denoted as $S_i$. Conversely, the personal parameters $\beta_i$ only perform multiple local iterations in each client $i$  and are not shared externally.

\begin{figure}
\centering
\includegraphics[width=.7\textwidth]{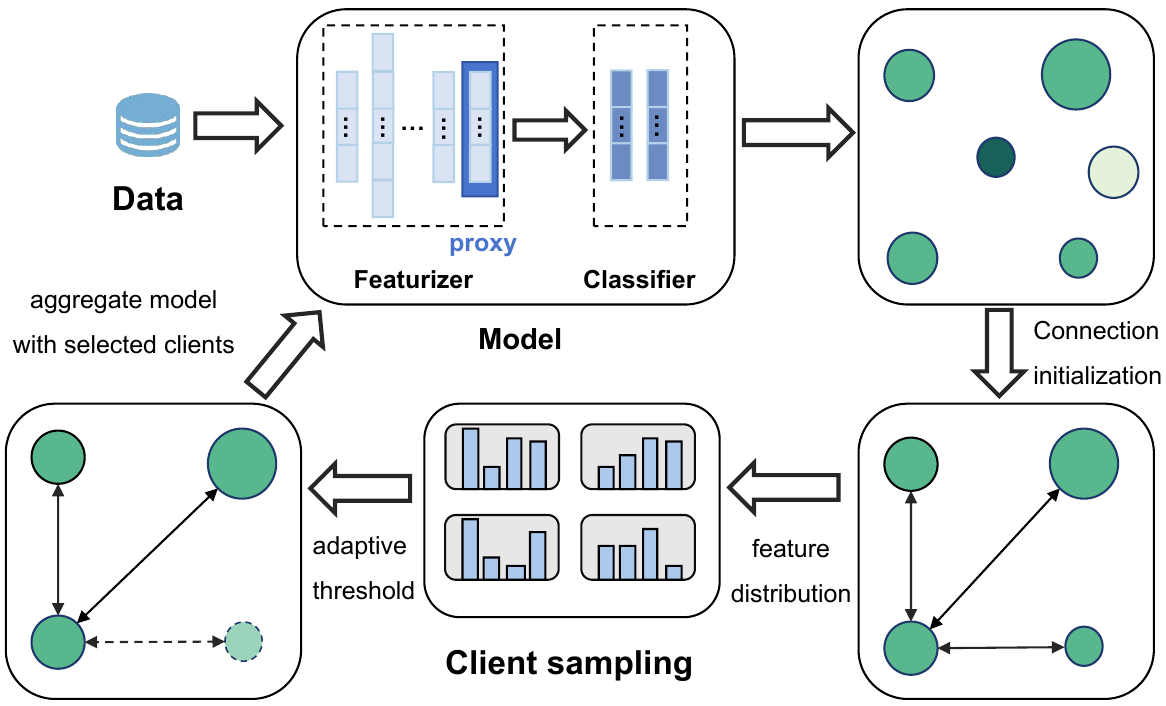}
\caption{\textbf{Illustration of the AFIND+ method.} 
The node represents the clients in the network, with different sizes and colors indicating variations in data distribution. First, the client identifies helpful neighbors with similar data distribution, represented by a similar feature proxy. Next, it filters neighbors using a threshold based on overall distribution and helpfulness. Finally, it aggregates the selected clients based on their contributions.
}
\label{illustration of workflow of AFIND}
\end{figure}

\subsection{Identify Right Neighbors}
\label{sec identify neighbors}

% In DFL, we aim to find helpful neighbors that improves client's performance, which is different from the CFL whose goal is selecting a subset to best approximate the full client participation~\citep{balakrishnan2021}. 
% Motivated by the fact that full clients' collaboration can be harmful while appropriate collaboration enhance performance, as shown in Figure~\ref{illustration of cooperation}, we aims to identify and find right neighbors that can benefit the client's performance from collaboration.
In DFL, the goal is to find neighbors that improve a client's performance, unlike CFL which aims to select a subset approximating full client participation~\citep{balakrishnan2021}. Motivated by the effectiveness of selective collaboration shown in Figure~\ref{illustration of cooperation}, our focus is on identifying beneficial neighbors for each client.

% We have asserted that the right neighbor selection means choosing those with similar data distributions, as evident in Figure~\ref{illustration of cooperation}. Similar observation found in \cite{sui2022find}. We first analyze the reasons behind this conclusion from a theoretical perspective, employing a coreset approach. Coresets are subsets of data points that efficiently approximate the original dataset, widely used in active learning~\cite{sener2017active}. The coreset definition is as follows:

We have posited that selecting the correct neighbors entails choosing those with similar data distributions, as depicted in Figure~\ref{illustration of cooperation}. A similar observation was made in \cite{sui2022find}. We begin by analyzing the rationale behind this assertion from a theoretical standpoint, employing a coreset approach. Coresets are subsets of data points that effectively approximate the original dataset and are commonly utilized in active learning~\cite{sener2017active}. The coreset definition is as follows:
\begin{definition}[Coreset]
    Given a dataset $\mathcal{D}$ and a loss function $L$, a coreset $\mathcal{C}$ is a subset of $\mathcal{D}$ such that for any model $\mathbf{x}$, the following inequality holds:
\begin{align}
    \|L(\mathbf{x}, \mathcal{D})-L(\mathbf{x}, \mathcal{C})\| \leq \epsilon \, ,
\end{align}
where $\epsilon$ is a small error tolerance.
\end{definition}

In DFL, the coreset of a target client is a set of clients that have similar model update performance.
To formulate this problem, we start by following the logic in \cite{mirzasoleiman2020coresets}.
Let $S$ be a subset of $m$ clients. Furthermore, assume that there is a mapping $\varsigma_\mathbf{x}: V \rightarrow S$ that for every possible value of optimization parameter $\mathbf{x} \in \mathcal{X}$ assigns every possible neighbor $i \in V$ to one of the elements $j$ in $S$, i.e., $\varsigma_\mathbf{x}(i)=j \in S$. Let $C_j=\{i \in[m] \mid \varsigma(i)= j\} \subseteq V$ be the set of neighbors that are assigned to $j \in S$, and $\gamma_j=\left|C_j\right|$ be the number of such neighbor nodes. Hence, $\left\{C_j\right\}_{j=1}^m$ form a partition of $V$. Then, for any $\mathbf{x}$ we can write the following gradient for client $i$:
% Given a subset $S$ and full set $V$ of client's neighbors,
% we define a mapping $\varsigma: V \rightarrow S$ such that the gradient information $\nabla F_i\left(\mathbf{x}\right)$ from client $i$ is approximated by the gradient information from a selected client $\varsigma_\mathbf{x}(i) \in S$ with model $\mathbf{x}$.
% For $i \in [m]$, let set $C_i \triangleq\{k \in [m] \mid \varsigma(k)=i\}$ be the set of clients who approximated client$-i$ (\textcolor{red}{i.e., the client coreset of client i} ) and $\gamma_i \triangleq\left|C_i\right|$. The gradient information of client $i$ can be written as:
\begin{small}
\begin{align}
\nabla F_i(\mathbf{x}) &= \nabla F_i(\mathbf{x}) - \sum\nolimits_{i \in V} \nabla F_i\left(\mathbf{x}\right) + \sum\nolimits_{i \in V} \left( \nabla F_i\left(\mathbf{x}\right) 
-\nabla F_{\varsigma_\mathbf{x}(i)}(\mathbf{x})+\nabla F_{\varsigma_\mathbf{x}(i)}(\mathbf{x}) \right) \\
&=\nabla F_i(\mathbf{x}) - \sum\nolimits_{i \in V} \nabla F_i\left(\mathbf{x}\right)+ \sum\nolimits_{i \in V} [\nabla F_i(\mathbf{x})  -\nabla F_{\varsigma_\mathbf{x}(i)}(\mathbf{x})] 
+\sum\nolimits_{j \in S} \gamma_j \nabla F_j(\mathbf{x}) \, .
\end{align}
\end{small}

Subtracting, and then taking the norm of both sides, we
get an upper bound on the error of estimating the client $i$'s gradient:
\begin{small}
\begin{align}
\label{derivation for coreset of DFL}
\|  \nabla F_i(\mathbf{x})-\sum\nolimits_{j \in S}  \gamma_j \nabla F_j(\mathbf{x}) \|^2 \leq  
 2\|\nabla F_i(\mathbf{x}) - \sum\nolimits_{i \in V} \nabla F_i\left(\mathbf{x}\right)\|^2+2 \sum\nolimits_{i \in V}\left\|\nabla F_i(\mathbf{x})-\nabla F_{\varsigma_\mathbf{x}(i)}(\mathbf{x})\right\|^2 \, ,
\end{align}
\end{small}
where the inequality follows from the triangle inequality and Jensen's inequality. The upper bound in Eq.~\eqref{derivation for coreset of DFL} is minimized when $\varsigma_\mathbf{x}$ assigns every $i\in V$ to an element in $S$ with the most gradient similarity at parameter $\mathbf{x}$, or the minimum Euclidean distance between the gradient vectors at $\mathbf{x}$. Therefore, based on the above conclusion that helpful neighbors should be clients with the most gradient similarity, we can formally define the coreset of DFL as follows:

\begin{definition}[Coreset of DFL]
Give a set of clients $[m]$ and loss function $L$, a coreset of client $i$ in DFL is a subset of $[m]$ such that 
\begin{align}
\label{def for DFL coreset}
    % \mathcal{C}_i = \argmin_{\mathcal{C}} \sum_{j\in \mathcal{C}}
    \sum\nolimits_{j\in \mathcal{C}_i^*}\|\nabla F_i(\mathbf{x}; {D}_i) - \nabla F_j(\mathbf{x}; {D}_j) \|^2 \leq \epsilon_i \, ,
\end{align}
where $\mathcal{C}_i^*$ represents the coreset of client $i$, and $\epsilon_i$ represents the error tolerance of client $i$.
\end{definition} 

The definition of coreset in DFL implies that, within a defined tolerance error, the coreset of client $i$ should be clients with similar model gradients. Motivated by the above theoretical findings and empirical observations ( Figure~\ref{illustration of cooperation} and \cite{onoszko2021decentralized} ) — where a smaller distance among clients' data distribution (reflected in model distance) improves performance—we propose a greedy sampling approach to identify the helpful clients. 

Clients with greater similarity in model gradients have a higher probability of being sampled. To reduce computation and communication costs, we use the output layer of the clients' featurizer as a feature proxy for the client's data distribution instead of the gradient for similarity calculation. It is worth noting that using the proxy does not cause additional information leakage compared to vanilla DFL approaches, as the proxy only constitutes one layer of the entire model. Specifically, the sampling probability for client $j$ to be sampled by client $i$ is calculated as:
\begin{align}
    p_{i,j}^t = \frac{\exp(\text{sim}(\mathbf{\phi}_i^t, \mathbf{\phi}_j^t)/\upsilon)}{\sum_{k\in \mathcal{B}_i^t}\exp(\text{sim}(\mathbf{\phi}_i^t, \mathbf{\phi}_k^t)/\upsilon)}  \, ,
\end{align}
where $\mathbf{\phi}_i^t$ represents the feature proxy of client $i$ at round $t$, and $\mathcal{B}_i^t$ represents the available neighbors of client $i$ at round $t$. Temperature $\upsilon$ controls the distribution shape. The 'sim' refers to the cosine similarity between two vectors.

In practice, only a subset of clients are available in each round; thus, we let
\begin{align}
    p_{i,j}^t = \frac{\exp(\text{sim}(\mathbf{\phi}_i^t, \mathbf{\phi}_j^t)/\upsilon)}{\sum_{k\in \mathcal{C}_i^t}\exp(\text{sim}(\mathbf{\phi}_i^t, \mathbf{\phi}_k^t)/\upsilon)} \left(1 - \sum\nolimits_{i \in \mathcal{B}_i^t/\mathcal{C}_i^t} p_{i,t}^t \right) \, ,
\label{practical sampling}
\end{align}
where $\mathcal{C}_i^t$ is the selected coreset for client $i$ at round $t$ and the multiplicative factor ensures that all probabilities sum to 1.

By $p_{i,j}^t$ we can identify helpful neighbors for client $i$, i.e., clients with a higher probability for collaboration. This is a greedy selection process. However, determining its termination conditions is challenging.
To address this, we propose an adaptive threshold to control the number of participating clients, as detailed in the next section.

% after assessing the similarity between client $i$ and its neighbors using features, we apply a threshold to determine whether a neighbor is considered beneficial and thus should be sampled.

\begin{algorithm}[ht]
    \caption{AFIND+}
    \label{alg:algorithm}
    \textbf{Input}:Number of clients $m$, global threshold $\tau$, learning rate $\eta_{\mathbf{w}}$ and $\eta_{\beta}$, number of local epoch $K_{\mathbf{w}}$ and $K_{\beta}$, total training rounds $T$. \\
    % \textbf{Parameter}: Optional list of parameters\\
    \textbf{Output}: Final model parameter $\mathbf{w}_i^T$ and $\beta_i^T$ \\
    \textbf{Initialize:} Shared model $\mathbf{w}_i^0$ and personalized model $\beta_i^0$. $p_{i,j} = \frac{1}{n_i}$ where $n_i = |\mathcal{B}_{i}^0|$ is the connected neighbors of client $i$, $\forall i \in [m]$.
    \begin{algorithmic}[1] %[1] enables line numbers
    \FOR{$t=0$ to $T-1$}
        \FOR{client $i$ in parallel}
            \IF{ $p_{i,j}^t \geq \theta_i^t$, for $j$ in $\mathcal{B}_i^t$}
                \STATE Add neighbor $j$ into the coreset $C_i^t$ for collaboration
            \ENDIF
            \FOR{client $j$ in $C_i^t$}
                \STATE Let $\mathbf{w}_j^{t,0} = w_i^t$, calculate $\mathbf{\phi}_j^t \leftarrow  [\nabla_{\mathbf{w}} F_j(\mathbf{w}_i^t,  \beta_j^{t};\xi_j)]_{\phi} $ and $F_j(\mathbf{w}_i^t,  \beta_j^{t};\xi_j)$;
                \STATE  \colorbox{shadecolor}{\emph{Local Update $\mathbf{w}_j$ and $\beta_j$}} and calculate  $\tilde{F}_j(\mathbf{w}_i^t,  \beta_j^{t};\xi_j)$ based on Eq.~\eqref{moving average};
                \STATE Communicate $\phi_j^t$, $w_j^{t+1}$, and $\tilde{F}_j(\mathbf{w}_i^t,  \beta_j^{t};\xi_j)$ to client $i$ 
            \ENDFOR
            \STATE $\mathbf{\phi}_i^t \leftarrow  [\nabla_{\mathbf{w}} F_i(\mathbf{w}_i^t,  \beta_i^{t};\xi_i)]_{\phi} $
            \STATE \colorbox{shadecolor}{\emph{Local Update $\mathbf{w}_i$ and $\beta_i$;}}
            % \FOR{$k=0$ to $K_{\mathbf{\beta}}-1$ }
            %     \STATE    Perform local parameter $\beta_i$ update:
            %     \STATE    $\beta_i^{t,k+1} = \beta_i^{t,k}-\eta_{\beta}\nabla_{\beta}F_i(\mathbf{w}_{i}^t,\beta_{i}^t; \xi_i) $
            % \ENDFOR
            % \STATE $\beta_i^{t+1} \leftarrow \beta_i^{t,K_{\mathbf{w}}}$
            % \FOR{$k=0$ to $K_{\mathbf{w}}-1$ }
            %     \STATE $\mathbf{w}_i^{t,k+1} = \mathbf{w}_i^{t,k+1} -\eta_{\mathbf{w}} \nabla_{\mathbf{w}} F_i(\mathbf{w}_{i}^t,\beta_{i}^t; \xi_i)$;
            % \ENDFOR 
            \STATE Aggregation: $\mathbf{w}_i^{t+1} = \sum_{j\in C_i^t} \varpi_{i,j}^t \mathbf{w}_j^{t+1}$, where $\varpi_{i,j}^t$ follows from Eq.~\eqref{eq agg};
            \STATE Update $p_{i,j}^t$ and $\theta_i^t$ based on Eq.~\eqref{practical sampling} and Eq.\eqref{threshold eq}, respectively;
        \ENDFOR
    \ENDFOR
\vspace{-.5em}
\begin{shaded}  
\vspace{-.5em}
    { \ ~~~~~~~~~~~~~~~ /* \emph{Local Update $\mathbf{w}$ and $\beta$} */}
            \FOR{$k=0$ to $K_{\mathbf{\beta}}-1$ }
                % \STATE    Perform local parameter $\beta_i$ update:
                \STATE    $\beta_i^{t,k+1} = \beta_i^{t,k}-\eta_{\beta}\nabla_{\beta}F_i(\mathbf{w}_{i}^t,\beta_{i}^{t,k}; \xi_i) $
            \ENDFOR
            \STATE $\beta_i^{t+1} \leftarrow \beta_i^{t,K_{\mathbf{w}}}$
            \FOR{$k=0$ to $K_{\mathbf{w}}-1$ }
                \STATE $\mathbf{w}_i^{t,k+1} = \mathbf{w}_i^{t,k} -\eta_{\mathbf{w}} \nabla_{\mathbf{w}} F_i(\mathbf{w}_{i}^{t,k},\beta_{i}^{t+1}; \xi_i)$;
            \ENDFOR 
\vspace{-.5em}
\end{shaded}
\vspace{-.5em}
    \end{algorithmic}
    \vspace{-.1em}
\end{algorithm}

\subsection{Adaptive Threshold for Flexible Neighbor Participation}
\label{sec adaptive threshold}
% We desire all the helpful clients with similar data can be selected for collaboration. However, as we discussed in \textbf{Challenge 2} of Section~\ref{sec intro}, fixed number of sampling may result in bad performance and fine-tune the participation number during training is unpractical. In addition, there are always some nodes that are challenging for the model to accurately cluster, especially during the early training stages~\citep{zheng2023adaptive}.

We aim to select all helpful clients for collaboration in each round. However, as noted in \textbf{Challenge 2}, we don't know the number of helpful neighbors in each round. Simply fixing the number of collaborators can lead to poor performance, and fine-tuning this number during training is impractical.

To this end, we propose an adaptive threshold that assesses the overall distribution similarity between a client and all its available neighbors, updating as training progresses. Intuitively, if a client's neighbors tend to be more similar, the similarity threshold should adaptively decrease, allowing for a larger sampling pool. Conversely, if the neighbors are less similar to the client, the threshold should increase, potentially reducing the number of selected clients.

Motivated by the principles from GNN-FD~\cite{zheng2023adaptive}, the threshold setting should correlate with the model’s confidence in the overall similarity level between the client and its neighbors, as well as the similarity to individual neighbors, mirroring the model's learning state. To determine the confidence $\tilde{C}_i^t$ of client $i$, we first calculate the entropy using the following equation:
\begin{align}
    h_i^t = -\sum\nolimits_{j\in C_i} e^t_{i,j}\log(e^t_{i,j}) \, ,
\end{align}
where $h_i^t$ represents of node $i$'s uncertainty and $e_{i,j}^t\in (0,1)$ represents the refined consine similarity that comes from $e_{i,j}^t=\frac{\text{sim}(\phi_i,\phi_j)+1}{2}$.
Normalize $H^t = (h_1^t, \cdots,h_m^t)$ to obtain $\tilde{H}^t$ as follows:
\begin{align}
% \tilde{H} = \frac{H -\min(H)}{\max{H}-\min(H)}
\tilde{H^t} = \sigma (H^t) \, ,
\end{align}
where $\sigma (X) = \frac{1}{1+e^{-X}}$ is the sigmoid function. 
$\tilde{H}_i^t \in (0,1)$ represents the $i$-th element of $\tilde{H}^t$.

A higher value of $\tilde{H}_i$ suggests that client $i$ is more challenging to distinguish, the neighbors with diverse distribution are lower. The confidence $\tilde{C}_{i}^t$ of client $i$ is then obtained by:
\begin{align}
    \tilde{C}_{i}^t = 1-\tilde{H}^t_i \, .
\end{align}
% The confidence $q_i$ of node $i$ is then obtained using the $\tilde{H}_i$.

With the confidence of the client given, we can now set the threshold $\theta_{i}$ for client $i$ as:
\begin{align}
\label{threshold eq}
    \theta_{i}^t = \tau \tilde{C}_{i}^t \, ,
\end{align}
where $\tau$ represents the global threshold, a constant. By comparing this threshold with $p_{i,j}^t$, we can determine the termination condition for the proposed greedy selection process and enable a flexible number of participating clients.

\begin{remark}
    Since $\tilde{C}_{i}^t\in (0,1)$, the threshold $\theta_{i}$ will not exceed $\tau$. Higher $\tilde{C}_{i}$ values indicate closer alignment of neighbor distributions with the client, resulting in a smaller threshold and a greater chance of neighbor participation, potentially increasing the number of participating neighbors.
\end{remark}

\subsection{Contribution-Awareness Aggregation}
\label{sec aggregation}

While we successfully identify and select friendly neighbors for collaboration, variations in data distributions among sampled users lead to differing contributions to cooperation. To utilize the selected neighbors' information more effectively, we introduce an aggregation strategy based on the Boltzmann distribution, using the loss of sampled clients to quantify their contributions. This aggregation strategy evaluates the impact of a client's distribution on the model $\mathbf{x}$, as expressed in the following formulation:
\begin{align}
    \varpi_{i,j}^t = \frac{e^{-F_j(\mathbf{w}_i,\beta_j;\xi_j)/T}}{Z} \, ,
\end{align}
where $T$ is the temperature and $Z= \sum_{i\in C_i} e^{-F_j(\mathbf{w}_i,\beta_j;\xi_j)/T}$ is the partition function. In this paper, we simply use $T=1$ to obtain satisfactory results.
% The temperature parameter $T$ controls the trade-off between exploration and exploitation, for example, a larger $T$ correspoinds to a larger    

To enhance performance by smoothing loss between clients, we propose a moving average of current and previous round losses:
\begin{align}
    \tilde{F}_j(\mathbf{w}_i^t,\beta_j^t;\xi_j) =  (1-\gamma)  F_j(\mathbf{w}_i^{t,K},\beta_j^{t,K};\xi_j) + \gamma F_j (\mathbf{w}_i^t,\beta_j^t;\xi_j) \, ,
    \label{moving average}
\end{align}
where $\gamma$ is a constant.
Correspondingly, the aggregation distribution should be:
\begin{align}
\label{eq agg}
    \varpi_{i,j}^t = \frac{e^{-\tilde{F}_j(\mathbf{w}_i^t,\beta_j^t;\xi_j)/T}}{Z} \, .
\end{align}

\begin{remark}
The Boltzmann distribution balances contributions from different neighbors probabilistically. Better-performing neighbors have a higher influence on the client model, but it also allows for contributions from less-performing devices.
\end{remark}

\vspace{-.5em}
\section{Convergence Analysis}
\vspace{-.2em}
\label{sec convergence}
% Beyond the convergence rate under common assumption, we further provide the convergence of objective under Polyak-Łojasiewicz condition.

To ease the theoretical analysis of our work, we use the following widely used assumptions:

\begin{assumption}[Smoothness]
\label{assumption 1}
    For each client $i=\{1, \ldots, m\}$, the function $F_i$ is continuously differentiable. There exist constants $L_\mathbf{w}, L_{\beta}, L_{\mathbf{w} \beta}, L_{\beta \mathbf{w}}$ such that for each client $i=\{1, \ldots, m\}$ :
$\nabla_{\mathbf{w}} f_i\left(\mathbf{w}_i, \beta_i\right)$ is $L_{\mathbf{w}}$-Lipschitz with respect to $\mathbf{w}$ and $L_{\mathbf{w} \beta}$-Lipschitz with respect to $\beta_i$;
$\nabla_{\beta} f_i\left(\mathbf{w}_i, \beta_i\right)$ is $L_{\beta}$-Lipschitz with respect to $\beta_i$ and $L_{\beta \mathbf{w}}$-Lipschitz with respect to $\mathbf{w}_i$.
\end{assumption}

\begin{assumption}[Bounded Variance]
\label{assumption 2}
The stochastic gradients $\nabla_w \hat{f}_i\left(w, \beta_i; \xi \right), \nabla_\beta \hat{f}_i\left(w, \beta_i; \xi\right)$ satisfy for all $i\in[m], w \in \mathbb{R}^{d_0}$, $\beta_i \in \mathbb{R}^{d_i}:$
\begin{align}
 \mathbb{E}\left[\left\|\nabla_w \hat{f}_i\left(w, \beta_i; \xi \right)-\nabla_w f_i\left(w, \beta_i\right)\right\|^2\right] \leq  
A_1\left\|\nabla_w f_i\left(w, \beta_i\right)\right\|^2+\sigma_{w}^2, \\
 \mathbb{E}\left[\left\|\nabla_\beta \hat{f}_i\left(w, \beta_i; \xi\right)-\nabla_\beta f_i\left(w, \beta_i\right)\right\|^2\right] \leq
A_2\left\|\nabla_\beta f_i\left(w, \beta_i\right)\right\|^2+\sigma_{\beta}^2,
\end{align}
for all $i \in[m]$, where $A_1, A_2, \sigma_{w}, \sigma_{\beta}$ are all positive constants.
\end{assumption}

\begin{assumption}[Bounded Dissimilarity]
\label{assumption 3}
     There is a positive constant $\lambda>0$ such that for all $w \in \mathbb{R}^{d_0}$ and $\beta_i \in \mathbb{R}^{d_i}, i \in[i]$, we have
\begin{align}
\frac{1}{m} \sum_{i=1}^m\left\|\nabla f_i\left(w, \beta_i\right)\right\|^2 \leq \lambda\left\|\frac{1}{m} \sum_{i=1}^m \nabla f_i\left(w, \beta_i\right)\right\|^2+\sigma_{G}^2 .
\end{align}
\end{assumption}

The above three assumptions are commonly used in both non-convex optimization and FL literature, see e.g.~\cite{sun2024adasam,yang2021achieving,koloskova2020unified,lin2021quasi,haddadpour2019convergence,hanzely2021personalized}. For Assumption~\ref{assumption 1}, to simplify the notation, we are making $L$ as a consistent upper bound for $L_{\mathbf{w}}$ and $L_{\beta}$.
For Assumption~\ref{assumption 3}, if all local loss functions are identical (homogeneous distribution), then we have $\sigma_{G} = 0$. 

Given the above assumptions, we can establish the following convergence rate of AFIND+ (Algorithm~\ref{alg:algorithm}) for general nonconvex objectives.

\begin{theorem}
\label{theorem 1}
    Under Assumptions 1 to 3, and let local learning rate $\eta_k=\eta$, for all $k \geq 0$, where $\eta$ is small enough to satisfy
$
\eta L \lambda\left(\frac{A_1}{m}+A_2+1\right)+\lambda \eta^2 L^2(K-1) K\left(A_1+1\right) -1\leq 0.
$
% Decentralized FL when sampling neighbor clients with AFIND sampling scheme satisfying Algorithm, 1  converges to a stationary point of $\nabla f$:
The convergence upper bound of Algorithm~\ref{alg:algorithm} AFIND+ satisfies:
\begin{align}
 \frac{1}{t} \sum_{t=0}^{T-1} \mathbb{E}\left[\left\|\frac{1}{m} \sum_{i=1}^m \nabla f_i\left(w^t, \beta_i^t\right) \right\|^2\right]  
\leq \frac{2F}{\eta T}+\Phi \, ,
\end{align}
where 
\begin{align}
\Phi = \eta L \lambda\left\{\left(\frac{A_1}{m}+A_2+1\right) \sigma_{G}^2+\frac{\sigma_{w}^2}{m}+\sigma_{\beta}^2\right\} 
 +\eta^2 L^2 \sigma_{G}^2(K-1)^2\left(A_1+1\right)+\eta^2 L^2 \sigma_{w}^2(K-1)^2 \, ,
\end{align}
where $F=\mathbb{E}\left[\frac{1}{m} \sum_{i=1}^m f_i\left(w^0, \beta_i^0\right)-f^*\right]$ 
 and $w^t:=\frac{1}{m} \sum_{i=1}^m w_i^t$ is a sequence of so-called virtual iterates.
\end{theorem}
% These variables have a significant influence on convergence bound, like variance bound $\sigma_{\mathbf{w}},\sigma_{\beta},\sigma_{G}$, and Lipschitz constant $L$.  

The variables, such as variance term $\sigma_{\mathbf{w}},\sigma_{\beta},\sigma_{G}$, and Lipschitz constant $L$, all control the convergence bound.
Let  $\eta = \mathcal{O}\left(\frac{1}{\sqrt{T}KL}\right)$, the convergence rate of AFIND+ is:
\begin{small}
\begin{align}
     \frac{1}{t} \sum_{t=0}^{T-1} \mathbb{E}\left[\left\|\frac{1}{m} \sum_{i=1}^m \nabla f_i\left(w^t, \beta_i^t\right) \right\|^2\right] \leq \mathcal{O}\left(\frac{1}{\sqrt{T}}+\frac{1}{T}\right) \, .
\end{align}
\end{small}
Due to space limitations, further details and complete proof are deferred to Appendix A.

\begin{table*}[!t]
 \centering
 \caption{\small \textbf{Performance improvement of AFIND+. } 
 We compare the performance of different sampling methods in DFL on CIFAR-10 and CIFAR-100 with both Dirichlet and Pathological distributions. The $\alpha$ represents the Dirichlet parameter and $c$ represents sampling classes of pathological partition. The data is divided into 100 clients, with 10 clients sampled in each round for baselines. Each experiment comprises 500 communication rounds, with the number of local epochs set to 5. We measure the average test accuracy of all clients in each communication round and report the best performance attained across all rounds. The results are then averaged over three seeds. We highlight the best results by using \textbf{bold font}.} 
 \fontsize{7.6}{11}\selectfont 
 \resizebox{1.\textwidth}{!}{%
  \begin{tabular}{l l l l l l l l l c c}
   \toprule
   \multirow{3}{*}{Algorithm} & \multicolumn{4}{c}{CIFAR-10 } & \multicolumn{4}{c}{CIFAR-100 }\\
   \cmidrule(lr){2-5} \cmidrule(lr){6-9}   
                    &\multicolumn{2}{c}{Dirichlet}   & \multicolumn{2}{c}{Pathological}    &\multicolumn{2}{c}{Dirichlet}   &\multicolumn{2}{c}{Pathological}  \\
      \cmidrule(lr){2-3} \cmidrule(lr){4-5} \cmidrule(lr){6-7} \cmidrule(lr){8-9}       
                    & $\alpha = 0.1$ & $\alpha=0.5$      & $c=2$ & $c=5$ & $\alpha = 0.1$ & $\alpha=0.5$      & $c=5$ & $c=10$ \\
   \midrule
Local                      & 48.44{\transparent{0.5}±0.44} & 62.36{\transparent{0.5}±0.51}  & 86.52{\transparent{0.5}±0.56} & 70.96{\transparent{0.5}±0.48} & 32.21{\transparent{0.5}±0.52} & 29.52{\transparent{0.5}±1.12} & 75.43{\transparent{0.5}±0.89} & 53.06{\transparent{0.5}±0.90}\\ 
    FedAvg             & 55.29{\transparent{0.5}±0.04} & 78.24{\transparent{0.5}±0.21} & 59.81{\transparent{0.5}±0.41} & 56.87{\transparent{0.5}±0.58} & 45.42{\transparent{0.5}±0.26} & 46.48{\transparent{0.5}±1.29} & 27.20{\transparent{0.5}±0.81} & 35.29{\transparent{0.5}±0.62}\\
    Gossip & 59.47{\transparent{0.5}±0.09} & 80.74{\transparent{0.5}±0.11} & 85.31{\transparent{0.5}±0.30} & 73.46{\transparent{0.5}±0.36} & 56.74{\transparent{0.5}±0.46} & 55.77{\transparent{0.5}±0.69} & 79.54{\transparent{0.5}±0.50} & 64.23{\transparent{0.5}±0.73}\\
    PENS          & 62.66{\transparent{0.5}±0.19} & 83.12{\transparent{0.5}±0.12} & 88.03{\transparent{0.5}±0.31} & 75.83{\transparent{0.5}±0.34} & 57.91{\transparent{0.5}±0.53} & 57.89{\transparent{0.5}±1.88} & 82.65{\transparent{0.5}±0.84} &67.03{\transparent{0.5}±1.21}\\ 
    FedeRiCo          & 66.31{\transparent{0.5}±0.07} & 84.50{\transparent{0.5}±0.10} & 88.26{\transparent{0.5}±0.32} & 76.82{\transparent{0.5}±0.14} & 57.83{\transparent{0.5}±0.79} & 57.65{\transparent{0.5}±1.03} & 82.79{\transparent{0.5}±1.45} & 68.22{\transparent{0.5}±1.26}\\ 
    AFIND+       & \textbf{71.89}{\transparent{0.5}±0.06} & \textbf{88.82}{\transparent{0.5}±0.16} & \textbf{91.13}{\transparent{0.5}±0.27} & \textbf{80.36}{\transparent{0.5}±0.32} & \textbf{62.11}{\transparent{0.5}±1.82} & \textbf{61.02}{\transparent{0.5}±1.29}& \textbf{84.87}{\transparent{0.5}±0.81} &\textbf{71.07}{\transparent{0.5}±1.02} \\ 
 \bottomrule
\end{tabular}
}
\label{table of performance comparison}
\end{table*}

\begin{table}
\centering
\caption{ 
\textbf{Ablation study for AFIND+.} (i) Identification: Greedy sampling using a fixed number of participation. (ii) Threshold: Enable a flexible number of participation by confidence level threshold. (iii) Aggregation: Contribution-based aggregation.
} 
\fontsize{7.6}{11}\selectfont 
\resizebox{.6\textwidth}{!}{%
  \begin{tabular}{l l l l l l l l l c c}
   \toprule
   \multirow{2}{*}{Algorithm} & \multicolumn{2}{c}{CIFAR-10} & \multicolumn{2}{c}{CIFAR-100}\\
   \cmidrule(lr){2-3} \cmidrule(lr){4-5} 
                    & $\alpha=0.1$ & $c=5$     &  $\alpha=0.1$ & $c=10$\\
   \midrule
FedAvg-FT  & 59.47{\transparent{0.5}±0.09} & 73.46{\transparent{0.5}±0.36} & 32.21{\transparent{0.5}±0.52} & 64.23{\transparent{0.5}±0.73} \\ 
\quad + \emph{i} & 65.57{\transparent{0.5}±0.12} & 76.96{\transparent{0.5}±0.53} & 57.25{\transparent{0.5}±0.99}  &68.10{\transparent{0.5}±1.47}  \\ 
\quad + \emph{i} + \emph{ii} & 68.89{\transparent{0.5}±0.08} & 78.95{\transparent{0.5}±0.20} & 58.02{\transparent{0.5}±0.83}  &69.60{\transparent{0.5}±0.91}  \\ 
\quad + \emph{i} + \emph{ii} + \emph{iii} &71.89{\transparent{0.5}±0.06} & 80.36{\transparent{0.5}±0.32} & 62.11{\transparent{0.5}±1.82}  &71.07{\transparent{0.5}±1.02} \\ 
\bottomrule
\end{tabular}
}
\label{table of ablation of algorithm}
\end{table}

\begin{figure}[t]
\centering
    \begin{center}
\centerline{\includegraphics[width=0.5\columnwidth]{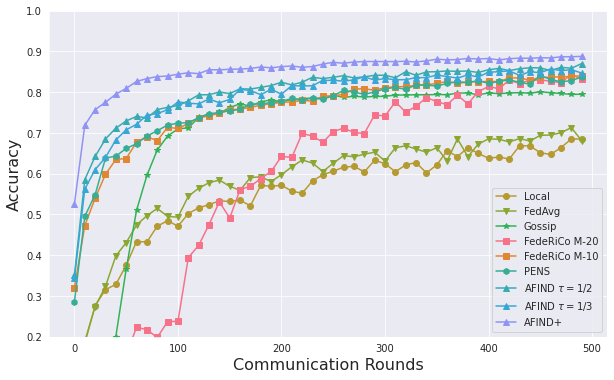}}
    \caption{\textbf{Comparison of convergence performance of DFL using different collaboration strategies.} AFIND refers to our method AFIND+ with uniform aggregation.}
    \label{Fig of convergence performance}
    \end{center}
\end{figure}

\begin{figure}[t]
    \centering
     \subfigure[Performance under Dirichlet non-IID \label{intro different cooperation}]{ \includegraphics[width=.4\textwidth,]{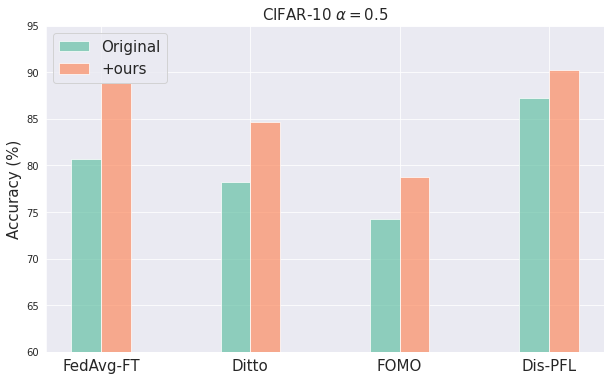}}
    \subfigure[Performance under pathological non-IID]{\includegraphics[width=.4\textwidth,]{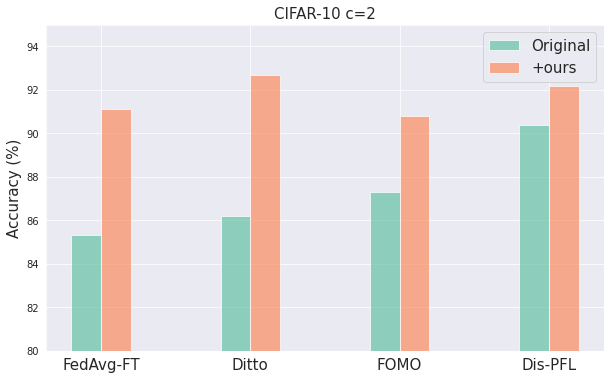}}
    \caption{Improving the performance of Decentralized PFL methods compatible with AFIND+.
}
    \label{performance of integrated with PFL}
\end{figure}

\section{Experiments}

\label{sec experiment}
In this section, we conduct experiments to validate the effectiveness of AFIND+. Detailed implementation and additional results are provided in Appendix due to space limitations.

\subsection{Experimental Setup}

\textbf{Datasets and Data partition.} We evaluate the performance of the proposed algorithm on three real datasets: FEMNIST~\citep{caldas2018leaf}, CIFAR-10, and CIFAR-100~\citep{krizhevsky2009learning}. We consider two different scenarios for simulating non-identical distribution (non-IID) of data across federated learning, following~\cite{dai2022dispfl}. \textbf{(1) Dirichlet Partition:} we partition the training data according to a Dirichlet distribution Dir($\alpha$) for each client following the same distribution. We use $\alpha =0.1$ and $\alpha=0.5$ for all datasets. The smaller the $\alpha$, the more heterogeneous the setting is. \textbf{(2) Pathological Partition:} we also evaluate with the pathological partition setup, as in \cite{zhang2020personalized}, where each client is assigned a limited number of classes at random. We sample 2 and 5 classes for CIFAR-10, and 5 and 10 classes for CIFAR-100. The number of sampled classes is denoted as $c$ in this paper.
\looseness=-1

\textbf{Baselines and Models.} We compare our methods with a diverse set of baselines, both from decentralized federated learning and personalized learning. A simple baseline called \emph{Local} is implemented with each client separately, only conducting local training on their own data. Additionally, we include the centralized method FedAvg. For the sampling method of DFL, we list the existing and SOTA sampling methods, including the random method gossip sampling~\citep{hegedHus2019gossip}, the performance-based method PENS~\citep{onoszko2021decentralized}, and FedeRiCo~\citep{sui2022find}. In addition, we test the proposed method's effectiveness in integrating with personalized FL algorithms, such as \citep{li2021ditto}, Fed-FT~\citep{cheng2021fine}, FOMO~\citep{zhang2020personalized}, and Dis-PFL~\citep{dai2022dispfl}.All accuracy results are reported as the mean and standard deviation across three different random seeds. Unless specified otherwise, AFIND refers to AFIND+ without aggregation, with a global threshold $\tau=0.5$ and momentum $\gamma=0.9$. We use a CNN model for FEMNIST and CIFAR-10, and ResNet-18 for CIFAR-100. More training settings are detailed in Appendix~\ref{app experiment setup}.

\subsection{Performance Evaluation}
\textbf{AFIND+ achieves notable accuracy improvement over other sampling methods and local training methods in DFL.} As shown in Table~\ref{table of performance comparison}, AFIND+ outperforms other baselines with the best accuracy across different datasets and data heterogeneity scenarios. Specifically, on CIFAR-10, AFIND+ achieves 71.89\% and 88.82\% accuracy for Dirichlet $\alpha=0.1$ and $\alpha=0.5$, respectively, which is 5.32\% and 4.32\% higher than the best-competing method, FedeRiCo. Similarly, for Pathological datasets, AFIND+ achieves 91.13\% and 80.36\% accuracy, outperforming the state-of-the-art FedeRiCo by 2.87\% and 3.54\% for $c=2$ and $c=5$, respectively. On CIFAR-100, AFIND+ shows at least a 3.37\% improvement for Dirichlet and a 2.08\% improvement for Pathological over state-of-the-art results. We attribute this improvement to AFIND+'s ability to identify the right collaborators, assign aggregation weights based on their contributions, and adaptively adjust the search space to include appropriate collaborators while excluding inappropriate ones, as demonstrated in the ablation study for AFIND+.

\textbf{AFIND+ converges faster than other sampling methods.} Figure~\ref{Fig of convergence performance} demonstrates the superior convergence performance of our sampling method on FEMNIST. Additionally, we observe that a smaller global threshold $\tau$ leads to slightly better performance, but the performance remains robust to the hyperparameter $\tau$. Moreover, when integrated with the importance-based aggregation method, AFIND+ shows much better performance than others, indicating greater efficiency in training. This implies that there are differences in contribution even among the correct collaborators being sampled, and our aggregation method has the ability to capture the contribution among collaborators.

\textbf{AFIND+ compatibility with PFL algorithms enhances performance in DFL.} Figure~\ref{performance of integrated with PFL} illustrates that AFIND+ is compatible with various PFL algorithms when applied in DFL. Despite variations, the improvement remains significant, resulting in an approximately $3\%$ increase in performance.

\textbf{All the components of AFIND+ are necessary and benefical.} We conduce the ablation study for AFIND+ in Table~\ref{table of ablation of algorithm}. It demonstrates that the identification of the right neighbors for collaboration, the adaptive threshold for client participation, and the contribution-aware aggregation method all contribute to improving the performance of DFL. FedAvg-FT refers to the personalized DFL algorithm with random sampling as the baseline, and our sampling method, aided by the threshold, significantly improves accuracy on the Dirichlet distribution by approximately $10\%$ for CIFAR-10 and  $20\%$ for CIFAR-100. Additionally, the proposed aggregation method further enhances performance by approximately $2\%$ for both CIFAR-10 and CIFAR-100.

Additional experiments demonstrate AFIND+'s superiority, including: (1) an ablation study for the threshold parameter $\tau$; (2) an ablation study for network topologies; and (3) a privacy evaluation, provided in Appendix~\ref{app additional experiments}.

\section{Conclusions, Limitations and Future Works}
\label{sec Conclusion}
In this paper, we introduce AFIND+, a novel DFL client collaboration optimization algorithm that identifies the right collaborators, allows flexible participation, and uses contribution-based aggregation. We empirically demonstrate its superiority over other DFL sampling algorithms and offer a convergence guarantee. However, ensuring its effectiveness in dynamic networks remains a challenge due to the varying availability of neighbors, which affects the search space of the sampling distribution. We leave this as a future research direction.

\clearpage
\newpage
\bibliographystyle{unsrt}  
\bibliography{references}  

\clearpage
\appendix
\onecolumn

% \onecolumn
{
 \hypersetup{linkcolor=black}
 \parskip=0em
 \renewcommand{\contentsname}{Contents of Appendix}
 \tableofcontents
 \addtocontents{toc}{\protect\setcounter{tocdepth}{3}}
}

\section{An extension of Related Works}
\label{sec app related works}

\paragraph{Decentralized Federated Learning (DFL)} aligns local models through peer-to-peer communication, with clients connecting solely to their neighbors~\citep{liu2022federated,kang2022blockchain,li2022effectiveness}. This structure minimizes single-point failure risks compared to Centralized Federated Learning (CFL), yet privacy and security concerns persist~\citep{witt2022decentral,gabrielli2023survey}. Extensive research addresses these challenges in DFL, focusing on privacy and diverse data distribution~\citep{wang2023enhancing,chen2022decentralized}. Foundational work by \cite{tan2022towards} and \cite{DBLP:journals/corr/McMahanMRA16} advances decentralized collaboration, emphasizing data privacy preservation in collaborative learning. Recent studies highlight the effectiveness of personalized models in DFL, given each client's local model~\citep{kulkarni2020survey,yu2023ironforge,wang2023confederated}. Despite these advancements, client sampling complexity in the DFL framework remains a prominent research area.

\paragraph{Client Selection in DFL} is a continuing challenge, with limited research focused on client sampling in Decentralized Federated Learning. \cite{hegedHus2019gossip} employs gossip sampling, which randomly selects clients. \cite{onoszko2021decentralized} suggests a heuristic, performance-based sampling method, where clients share model parameters with others to evaluate loss values; a lower loss value indicates a more preferable client for sampling. \cite{sui2022find} utilizes expectation maximization in a personalized DFL algorithm, selecting neighbors based on the exponent of negative loss. While these studies agree that selecting neighbors with similar data distributions enhances DFL performance, our approach is distinct. We propose using client features as proxies for distribution similarity. Moreover, unlike these studies which focus on fixed-number sampling, we introduce an adaptive threshold for neighbor selection, allowing for a variable number of participating neighbors.

\paragraph{DFL others.}
In addition to client sampling algorithms, FL faces other challenges such as privacy preservation \citep{wang2023enhancing, zhou2023decentralized, chen2023privacy} and communication efficiency \citep{chai2023communication, almanifi2023communication, paragliola2022definition}. We show that AFIND+ is compatible with existing privacy and communication enhancement methods, as shown in Section~\ref{sec experiment}.

\section{Theoretical Analysis and Proof}
\label{Sec app analysis}
In this section, we provide the analysis of Theorem 3.5, i.e., the convergence analysis.
\subsection{Useful Lemmas}
In this subsection, we will present some useful lemmas before giving the complete convergence proof.

We start by introducing additional notation, following form \cite{hanzely2021personalized}. We set $t_p=p \cdot K$, where $K \in \mathbb{N}^{+}$is the length of the averaging period. Let $t_p=p K+K-1=t_{p+1}-1=v_p$. Denote the total number of iterations as $T$ and assume that $T=k_{\bar{p}}$ for some $\bar{p} \in \mathbb{N}^{+}$. The final result is set to be that $\hat{w}=w^T$ and $\hat{\beta}_m=\beta_i^T$ for all $i \in[m]$. We assume that the solution to (2) is $w^*, \beta_1^*, \ldots, \beta_i^*$ and that the optimal value is $f^*$. Let $w^t=\frac{1}{m} \sum_{i=1}^m w_i^t$ for all $t$. Note that this quantity will not be actually computed in practice unless $t=t_p$ for some $p \in \mathbb{N}$, where we have $w^{t_p}=w_i^{t_p}$ for all $i \in[m]$. In addition, let $\xi_i^t=\left\{\xi_{1, i}^t, \xi_{2, i}^t, \ldots, \xi_{B, i}^t\right\}$ and $\xi^t=\left\{\xi_1^t, \xi_2^t, \ldots, \xi_m^t\right\}$.
% Let $x_i=\left(\left(w_m\right)^{\top},\left(\beta_i\right)^{\top}\right)^{\top}, x_i^t=\left(\left(w_i^t\right)^{\top},\left(\beta_i^t\right)^{\top}\right)^{\top}, x_m^*=\left(\left(w^*\right)^{\top},\left(\beta_i^*\right)^{\top}\right)^{\top}$ and $\hat{x}_i^t=\left(\left(w^k\right)^{\top},\left(\beta_i^t\right)^{\top}\right)^{\top}$.
Let
\begin{align}
g_i^t=\frac{1}{B} \nabla F_i\left(w_i^t, \beta_i^t ; \xi_i^t\right),
\end{align}
where
$
\nabla F_i\left(w_i^t, \beta_i^t ; \xi_i^t\right)=\sum_{j=1}^B \nabla F_i\left(w_i^t, \beta_i^t ; \xi_{j, i}^t\right) .
$

When the gradient is unbiased, we get
\begin{align}
\mathbb{E}\left[g_i^t\right]=\nabla F_i\left(w_i^t, \beta_i^t\right) .
\end{align}

Let
$
g_{i, 1}^t=\frac{1}{B} \nabla_w F_i\left(w_i^t, \beta_i^t ; \xi_i^t\right), \quad g_{i, 2}^t=\frac{1}{B} \nabla_{\beta_i} F_i\left(w_i^t, \beta_i^t ; \xi_i^t\right),
$
so that $g_i^t=\left(\left(g_{i, 1}^t\right)^{\top},\left(g_{i, 2}^t\right)^{\top}\right)^{\top}$. 

Thus, the parameter is updated by :
\begin{align}
\left(w_i^{t+1}, \beta_i^{t+1}\right)=\left(w_i^t, \beta_i^t\right)-\eta_t g_i^t .
\end{align}

Then we define
\begin{align}
h^t=\frac{1}{m} \sum_{t=1}^m g_{i, 1}^t, \quad V^t=\frac{1}{m} \sum_{i=1}^m\left\|w_i^t-w^t\right\|^2 .
\end{align}

Then $w^{t+1}=w^t-\eta_t h^t$ for all $t$.

We denote the Bregman divergence associated with $F_i$ for $x_i$ and $\bar{x}_i$ as
\begin{align}
D_{F_i}\left(x_i, \bar{x}_i\right):=f_i\left(x_i\right)-f\left(\bar{x}_i\right)-\left\langle\nabla f_i\left(\bar{x}_i\right), x_i-\bar{x}_i\right\rangle .
\end{align}

Finally, we define the sum of residuals as
\begin{align}
r^t=\left\|w^t-w^*\right\|^2+\frac{1}{m} \sum_{i=1}^m\left\|\beta_i^t-\beta_i^*\right\|^2=\frac{1}{m} \sum_{i=1}^m\left\|\hat{x}_i^t-x_i^*\right\|^2
\end{align}
and let $\sigma_{\text {G }}^2=\frac{1}{m} \sum_{i=1}^m\left\|\nabla f_i\left(x_i^*\right)\right\|^2$.

Following the standard results in \cite{nesterov2018lectures}, we present the following proposition:

\begin{proposition}
If the function $f$ is differentiable and $L$-smooth, then
\begin{align}
\label{eq lem 1}
f(x)-f(y)-\langle\nabla f(y), x-y\rangle \leq \frac{L}{2}\|x-y\|^2 .
\end{align}
If $f$ is also convex, then
\begin{align}
\|\nabla f(x)-\nabla f(y)\|^2 \leq 2 L D_f(x, y) \forall x,y.
\end{align}

For all vectors $x, y$, we have

\begin{align}
2\langle x, y\rangle & \leq \xi\|x\|^2+\xi^{-1}\|y\|^2, \quad \forall \varepsilon>0, \\
-\langle x, y\rangle & =-\frac{1}{2}\|x\|^2-\frac{1}{2}\|y\|^2+\frac{1}{2}\|x-y\|^2 .
\label{eq prosion 2}
\end{align}

For vectors $v_1, v_2, \ldots, v_n$, by Jensen's inequality and the convexity of the map: $x \mapsto\|x\|^2$, we have
\begin{align}
\label{Jensen ineq}
\left\|\frac{1}{n} \sum_{i=1}^n v_i\right\|^2 \leq \frac{1}{n} \sum_{i=1}^n\left\|v_i\right\|^2
\end{align}
\end{proposition}

Then, we provide some useful lemmas:
\begin{lemma}
\label{lemma 1}
Suppose Assumption 1 holds. Given $\left\{x_i^t\right\}_{i \in[m]}$, we have:

\begin{align}
& \mathbb{E}\left[\frac{1}{m} \sum_{i=1}^m F_i\left(\hat{x}_m^{k+1}\right)\right] -\frac{1}{m} \sum_{i=1}^m F_i\left(\hat{x}_i^t\right) \\
& \leq-\eta \left\langle\frac{1}{m} \sum_{i=1}^m \nabla_w F_i\left(\hat{x}_i^t\right), \frac{1}{m} \sum_{i=1}^m \nabla_w F_i\left(x_i^t\right)\right\rangle \\
& -\frac{\eta}{m} \sum_{i=1}^m\left(\nabla_{\beta_i} F_i\left(\hat{x}_i^t\right), \nabla_{\beta_i} F_i\left(x_i^t\right)\right\rangle \\
& +\frac{\eta^2 L}{2} \mathbb{E}\left[\left\|h^t\right\|^2\right]+\frac{\eta^2 L}{2 m} \sum_{i=1}^m \mathbb{E}\left[\left\|g_{i, 2}^t\right\|^2\right] \, ,
\end{align}
where the expectation is taken with respect to the randomness in $\xi^t$.
\end{lemma}

\begin{proof}
 By the $L$-smoothness assumption on $F_i(\cdot)$ and \eqref{eq lem 1}, we have
\begin{align}
F_i\left(\hat{x}_m^{k+1}\right)-&F_i\left(\hat{x}_i^t\right)-\left\langle\nabla F_i\left(\hat{x}_i^t\right), \hat{x}_m^{k+1}-\hat{x}_i^t\right\rangle \notag \\
&\leq \frac{L}{2}\left\|\hat{x}_m^{k+1}-\hat{x}_i^t\right\|^2 \, .
\end{align}

Thus, we have
\begin{align}
&F_i\left(\hat{x}_m^{k+1}\right)-F_i\left(\hat{x}_i^t\right) \leq-\eta\left\langle\nabla_w F_i\left(\hat{x}_i^t\right), h^t\right\rangle \notag \\
&-\eta\left\langle\nabla_{\beta_i} F_i\left(\hat{x}_i^t\right), g_{i, 2}^t\right\rangle+\frac{\eta^2 L}{2}\left\|h^t\right\|^2+\frac{\eta^2 L}{2}\left\|g_{i, 2}^t\right\|^2 \, ,
\end{align}
which implies that
\begin{align}
& \frac{1}{m} \sum_{i=1}^m F_i\left(\hat{x}_m^{k+1}\right)-\frac{1}{m} \sum_{i=1}^m F_i\left(\hat{x}_i^t\right) \leq \notag\\
&-\eta\left\langle\frac{1}{m} \sum_{i=1}^m \nabla_w F_i\left(\hat{x}_i^t\right), h^t\right\rangle-\frac{\eta}{m} \sum_{i=1}^m\left\langle\nabla_{\beta_i} F_i\left(\hat{x}_i^t\right), g_{i, 2}^t\right\rangle \notag \\
&+\frac{\eta^2 L}{2} \|h^t \|^2 +\frac{\eta^2 L}{2 m} \sum_{i=1}^m\left\|g_{i, 2}^t\right\|^2 \, .
\end{align}

The result follows by taking the expectation with respect to the randomness in $\xi^t$, while keeping the other quantities fixed.
\end{proof}

\begin{lemma}
\label{lemma 2}
Suppose Assumptions 2 and 3 hold. Given $\left\{x_i^t\right\}_{i \in[m]}$, we have

\begin{align}
&\mathbb{E}\left[\left\|h^t\right\|^2\right] +\frac{1}{m} \sum_{i=1}^m \mathbb{E}\left[\left\|g_{i, 2}^t\right\|^2\right] \\
& \leq\left(\frac{A_1}{m}+A_2+1\right) \frac{1}{m} \sum_{i=1}^m\left\|F_i\left(x_i^t\right)\right\|^2 +\frac{\sigma_w^2}{m}+\sigma_{\beta}^2 \\
& \leq \lambda\left(\frac{A_1}{m}+A_2+1\right)\left\|\frac{1}{m} \sum_{i=1}^m \nabla F_i\left(x_i^t\right)\right\|^2 \notag \\
& +\left(\frac{A_1}{m}+A_2+1\right) \sigma_{G}^2+\frac{\sigma_w^2}{m}+\sigma_{\beta}^2 \, ,
\end{align}
where the expectation is taken only with respect to the randomness in $\xi^t$.
\end{lemma}

\begin{proof}
Note that

\begin{align}
\mathbb{E}\left[\left\|h^t\right\|^2\right]& =\mathbb{E}\left[\left\|\frac{1}{m} \sum_{i=1}^m g_{i, 1}^t\right\|^2\right] \\
& \stackrel{(a)}{=} \mathbb{E}\left[\left\|\frac{1}{m} \sum_{i=1}^m\left(g_{i, 1}^t-\nabla_w F_i\left(x_i^t\right)\right)\right\|^2\right] \notag \\
&+\left\|\frac{1}{m} \sum_{i=1}^m \nabla_w F_i\left(x_i^t\right)\right\|^2 \\
& \stackrel{(b)}{=} \frac{1}{m^2} \sum_{i=1}^m \mathbb{E}\left[\left\|g_{i, 1}^t-\nabla_w F_i\left(x_i^t\right)\right\|^2\right] \notag \\
&+\left\|\frac{1}{m} \sum_{i=1}^m \nabla_w F_i\left(x_i^t\right)\right\|^2 \\
& \stackrel{(c)}{\leq} \frac{1}{m^2} \sum_{\mathrm{i}=1}^m\left(A_1\left\|\nabla F_{\mathrm{i}}\left(x_{\mathrm{i}}^{\mathrm{t}}\right)\right\|^2+\sigma_w^2\right) \notag \\
&+\left\|\frac{1}{m} \sum_{\mathrm{i}=1}^m \nabla_{\mathrm{w}} F_{\mathrm{i}}\left(x_{\mathrm{i}}^t\right)\right\|^2 \\
& \stackrel{(d)}{\leq} \frac{A_1}{m^2} \sum_{\mathrm{i}=1}^m\left\|\nabla F_{\mathrm{i}}\left(x_{\mathrm{i}}^t\right)\right\|^2+\frac{\sigma_w^2}{m} \notag \\
&+\frac{1}{m} \sum_{\mathrm{i}=1}^m\left\|\nabla_{\mathrm{w}} f_{\mathrm{i}}\left(x_{\mathrm{i}}^t\right)\right\|^2 \, ,
\end{align}

where (a) is attributed to the unbiased nature of $g_{i, 1}^t$, (b) is a consequence of the independence of $\xi_1^k, \xi_2^k, \ldots, \xi_i^t$, (c) is based on Assumption 2, and (d) follows from \eqref{Jensen ineq}.

We also have
\begin{align}
&\frac{1}{m} \sum_{i=1}^m \mathbb{E}\left[\left\|g_{i, 2}^t\right\|^2\right]  =\frac{1}{m} \sum_{i=1}^m \mathbb{E}\left[\left\|g_{i, 2}^t-\nabla_{\beta_i} F_i\left(x_i^t\right)\right\|^2\right] \notag \\
&+\frac{1}{m} \sum_{i=1}^m\left\|\nabla_{\beta_i} F_i\left(x_i^t\right)\right\|^2 \\
& \leq \frac{A_2}{m} \sum_{i=1}^m \left\| \nabla F_i\left(x_i^t\right)\left\|^2+\sigma_{\beta}^2+\frac{1}{m} \sum_{i=1}^m\right\| \nabla_{\beta_i} F_i\left(x_i^t\right)\right. \|^2 .
\end{align}

\end{proof}

% The lemma follows by combining the two inequalities.
\begin{lemma}
\label{lemma 3}
Under Assumption 1, we have
\begin{align}
&-\eta\left\langle\frac{1}{m} \sum_{i=1}^m \nabla_w F_i\left(\hat{x}_i^t\right), \frac{1}{m} \sum_{i=1}^m \nabla_w F_i\left(x_i^t\right)\right\rangle \notag \\
&-\frac{\eta}{m} \sum_{i=1}^m\left\langle\nabla_{\beta_i} F_i\left(\hat{x}^k\right), \nabla_{\beta_i} F_i\left(x_i^t\right)\right\rangle \\
& \leq-\frac{\eta}{2}\left\|\frac{1}{m} \sum_{i=1}^m \nabla F_i\left(\hat{x}_i^t\right)\right\|^2-\frac{\eta}{2}\left\|\frac{1}{m} \sum_{i=1}^m \nabla F_i\left(x_i^t\right)\right\|^2\notag \\
&+\frac{\eta L^2}{2} V^t
\end{align}

\end{lemma}

\begin{proof}
By \eqref{eq prosion 2}, we have
\begin{align}
& -\eta\left\langle\frac{1}{m} \sum_{i=1}^m \nabla_w F_i\left(\hat{x}_i^t\right), \frac{1}{m} \sum_{i=1}^m \nabla_w F_i\left(x_i^t\right)\right\rangle \\
& =-\frac{\eta}{2} \left\|  \frac{1}{m} \sum_{i=1}^m \nabla_w F_i\left(\hat{x}_i^t\right)\right\|^2-\frac{\eta}{2}\left\| \frac{1}{m} \sum_{i=1}^m \nabla_w F_i\left(x_i^t\right) \right\|^2 
\notag \\
&+\frac{\eta}{2}\left\|\frac{1}{m} \sum_{i=1}^m\left(\nabla_w F_i\left(\hat{x}_i^t\right)-\nabla_w F_i\left(x_i^t\right)\right)\right\|^2 \\
&\leq-\frac{\eta}{2} \|  \frac{1}{m} \sum_{i=1}^m \nabla_w F_i\left(\hat{x}_i^t\right)\left\|^2-\frac{\eta}{2}\right\| \frac{1}{m} \sum_{i=1}^m \nabla_w F_i\left(x_i^t\right) \|^2 \notag \\
& +\frac{\eta}{2 m} \sum_{i=1}^m\left\|\nabla_w F_i\left(\hat{x}_i^t\right)-\nabla_w F_i\left(x_i^t\right)\right\|^2 \, ,
\end{align}
where the last inequality follows from \eqref{Jensen ineq}. 
In addition,

\begin{align}
&-\eta\left\langle\nabla_{\beta_i} F_i\left(\hat{x}^k\right), \nabla_{\beta_i}\right.  \left.F_i\left(x_i^t\right)\right\rangle 
 =  -\frac{\eta}{2}\left\|\nabla_{\beta_i} F_i\left(\hat{x}_i^t\right)\right\|^2\notag \\
&-\frac{\eta}{2}\left\|\nabla_{\beta_i} F_i\left(x_i^t\right)\right\|^2+\frac{\eta }{2}\left\|\nabla_{\beta_i} F_i\left(\hat{x}_i^t\right)-\nabla_{\beta_i} F_i\left(x_i^t\right)\right\|^2 .
\end{align}

Therefore, we have
\begin{align}
&-\frac{\eta}{m}\left\langle\nabla_{\beta_i} F_i\left(\hat{x}^k\right), \nabla_{\beta_i} F_i\left(x_i^t\right)\right\rangle \notag \\
= &-\frac{\eta}{2 m} \sum_{i=1}^m\left\|\nabla_{\beta_i} F_i\left(\hat{x}_i^t\right)\right\|^2-\frac{\eta}{2 m} \sum_{i=1}^m \|\left.\nabla_{\beta_i} F_i\left(x_i^t\right)\right|^2 \notag \\
& +\frac{\eta}{2 m} \sum_{i=1}^m\left\|\nabla_{\beta_i} F_i\left(\hat{x}_i^t\right)-\nabla_{\beta_i} F_i\left(x_i^t\right)\right\|^2 \\
\leq & -\frac{\eta}{2}\left\|\frac{1}{m} \sum_{i=1}^m \nabla_{\beta_i} F_i\left(\hat{x}_i^t\right)\right\|^2-\frac{\eta}{2}\left\|\frac{1}{m} \sum_{i=1}^m \nabla_{\beta_i} F_i\left(x_i^t\right)\right\|^2 \notag \\
& +\frac{\eta}{2 m} \sum_{i=1}^m\left\|\nabla_{\beta_i} F_i\left(\hat{x}_i^t\right)-\nabla_{\beta_i} F_i\left(x_i^t\right)\right\|^2 \, .
\end{align}

Combining the above equations, we have

\begin{align}
&-\eta  \left\langle\frac{1}{m} \sum_{i=1}^m \nabla_w F_i\left(\hat{x}_i^t\right), \frac{1}{m} \sum_{i=1}^m \nabla_w F_i\left(x_i^t\right)\right\rangle \notag \\
&-\frac{\eta}{m} \sum_{i=1}^m\left(\nabla_{\beta_i} F_i\left(\hat{x}_i^t\right), \nabla_{\beta_i} F_i\left(x_i^t\right)\right\rangle \\
\leq & -\frac{\eta}{2}\left\|\frac{1}{m} \sum_{i=1}^m \nabla F_i\left(\hat{x}_i^t\right)\right\|^2-\frac{\eta}{2}\left\|\frac{1}{m} \sum_{i=1}^m \nabla F_i\left(x_i^t\right)\right\|^2 \notag \\
& +\frac{\eta}{2 m} \sum_{i=1}^m\left\|\nabla F_i\left(\hat{x}_i^t\right)-\nabla F_i\left(x_i^t\right)\right\|^2 \\
\stackrel{(e)}{\leq}  & -\frac{\eta}{2}\left\|\frac{1}{m} \sum_{i=1}^m \nabla F_i\left(\hat{x}_i^t\right)\right\|^2-\frac{\eta}{2}\left\|\frac{1}{m} \sum_{i=1}^m \nabla F_i\left(x_i^t\right)\right\|^2\notag \\
&+\frac{\eta L^2}{2 m} \sum_{i=1}^m\left\|w_i^t-w^k\right\|^2 \\
= & -\frac{\eta}{2}\left\|\frac{1}{m} \sum_{i=1}^m \nabla F_i\left(\hat{x}_i^t\right)\right\|^2-\frac{\eta}{2}\left\|\frac{1}{m} \sum_{i=1}^m \nabla F_i\left(x_i^t\right)\right\|^2 \notag \\
&+\frac{\eta L^2}{2} V^t \, ,
\end{align}
where (e) comes from Assumption 1.
\end{proof}

% \begin{lemma}
% Under Assumptions 4 and 7 , we have
% $$
% -\eta\left\langle\frac{1}{m} \sum_{i=1}^m \nabla_w F_i\left(\hat{x}_i^t\right), \frac{1}{m} \sum_{i=1}^m \nabla_w F_i\left(x_i^t\right)\right\rangle-\frac{\eta}{m} \sum_{i=1}^m\left\langle\nabla_{\beta_i} F_i\left(\hat{x}^k\right), \nabla_{\beta_i} F_i\left(x_i^t\right)\right\rangle
% $$

% $$
% \leq-\eta \mu\left(\frac{1}{m} \sum_{i=1}^m \nabla F_i\left(\hat{x}_i^t\right)-f^*\right)-\frac{\eta}{2}\left\|\frac{1}{m} \sum_{i=1}^m \nabla F_i\left(x_i^t\right)\right\|^2+\frac{\eta L^2}{2} V^t .
% $$
% \end{lemma}
% The proof follows directly from Lemma 10 and Assumption 7.

\begin{lemma}
\label{lemma 4}
Suppose Assumptions 2 and 3 hold. For $t_p+1 \leq t \leq \mathrm{v}_{p}$, we have
\begin{align}
&\mathbb{E}\left[V^t\right] \notag \\
\leq &\lambda(K-1)\left(A_1+1\right) \sum_{\tau=t_p}^{t-1} \eta_{\tau}^2 \mathbb{E}\left[\left\|\frac{1}{m} \sum_{i=1}^m \nabla F_i\left(x_i^{\tau}\right)\right\|^2\right] \\
+&\sigma_{G}^2(K-1)\left(A_1+1\right) \sum_{\tau=t_p}^{t-1} \eta_{\tau}^2+\frac{\sigma_w^2(K-1)}{B} \sum_{\tau=t_p}^{t-1} \eta_{\tau}^2 .
\end{align}
Note that $V^{t_p}=0$.
\end{lemma}

\begin{proof}
Note that $w^{t_p}=w_i^{t_p}$ for all $i \in[m]$. Thus, for $t_p+1 \leq t \leq v_{p}$, we have
\begin{align}
&\left\|w_i^t-w^k\right\|^2=\left\|w_i^t-\sum_{\tau=t_p}^{t-1} \eta_\tau g_{i, 1}^\tau-w^{t_p}-\sum_{\tau=t_p}^{t-1} \eta_\tau h^{\tau}\right\|^2 \notag \\
&=\left\|\sum_{\tau=t_p}^{t-1} \eta_{\tau} g_{i, 1}^{\tau}-\sum_{\tau=t_p}^{t-1} \eta_{\tau} h^{\tau}\right\|^2 .
\end{align}

Since
\begin{align}
\frac{1}{m} \sum_{i=1}^m \sum_{\tau=t_p}^{t-1} \eta_{\tau} g_{i, 1}^{\tau}=\sum_{\tau=t_p}^{t-1} \eta_{\tau} h^{\tau}.
\end{align}

Then we have

\begin{align}
& \frac{1}{m} \sum_{i=1}^m\left\|w_i^t-w^t\right\|^2  \\
=&\frac{1}{m} \sum_{i=1}^m\left\|\sum_{\tau=t_p}^{t-1} \eta_{\tau} g_{i, 1}^\tau-\sum_{\tau=t_p}^{t-1} \eta_{\tau} h^{\tau}\right\|^2  \\
=&\frac{1}{m} \sum_{i=1}^m\left\|\sum_{\tau=t_p}^{t-1} \eta_{\tau} g_{i, 1}^{\tau}\right\|^2-\left\|\sum_{\tau=t_p}^{t-1} \eta_{\tau} h^{\tau}\right\|^2  \\
\leq & \frac{1}{m} \sum_{i=1}^m\left\|\sum_{\tau=t_p}^{t-1} \eta_{\tau} g_{i, 1}^{\tau} \right\|^2  \leq \frac{t-t_p}{m} \sum_{i=1}^m \sum_{\tau=t_p}^{t-1} \eta_{\tau}^2\left\|g_{i, 1}^{\tau}\right\|^2 \\
\leq &\frac{K-1}{m} \sum_{i=1}^m \sum_{\tau=t_p}^{t-1}\eta_{\tau}^2\left\|g_{i, 1}^{\tau}\right\|^2 \, .
\label{equation 3 in app}
\end{align}

Given $\left\{x_i^t\right\}_{i \in[m] }$, we have

\begin{align}
&\mathbb{E}\left[\frac{1}{m} \sum_{i=1}^m\left\|g_{i, 1}^t\right\|^2\right]  =\frac{1}{m} \sum_{i=1}^m \mathbb{E}\left[\left\|g_{i, 1}^t\right\|^2\right] \\
& =\frac{1}{m} \sum_{i=1}^m \mathbb{E}\left[\left\|g_{i, 1}^t-\nabla_w F_i\left(x_i^t\right)\right\|^2\right] \notag \\
&+\frac{1}{m} \sum_{i=1}^m  \| \nabla_w F_i\left(x_i^t\right) \|^2 \\
& \leq \frac{1}{m} \sum_{i=1}^m\left[\left(A_1+1\right) \left\|\nabla F_i\left(x_i^t\right)\right\|^2+\sigma_w^2\right] \notag \\
&+\frac{1}{m} \sum_{i=1}^m\left\|\nabla F_i\left(x_i^t\right)\right\|^2 \\
& =\frac{A_1+1}{m} \sum_{i=1}^m\left\|\nabla F_i\left(x_i^t\right)\right\|^2+\sigma_w^2 .
\end{align}

where the expectation is taken with respect to the randomness in $\xi^t$. Thus, by the independence of $\xi^{(1)}, \xi^{(2)}, \ldots, \xi^t$ and taking an unconditional expectation on both sides of \eqref{equation 3 in app}, we have

\begin{align}
&\mathbb{E}\left[V^t\right]  =(K-1) \sum_{\tau=t_p}^{t-1} \eta_{\tau}^2 \mathbb{E}\left[\mathbb{E}\left[\frac{1}{m} \sum_{i=1}^m\left\|g_{i, 1}^{\tau}\right\|^2\right]\right] \\
& \leq(K-1)\left(A_1+1\right) \sum_{\tau=t_p}^{t-1} \eta_{\tau}^2 \mathbb{E}\left[\frac{1}{m} \sum_{i=1}^m\left\|\nabla F_i\left(x_i^{\tau}\right)\right\|^2\right] \notag \\
&+{(K-1) \sigma_w^2} \sum_{\tau=t_p}^{t-1} \eta_{\tau}^2 \\
& \leq \lambda(K-1)\left(A_1+1\right) \sum_{\tau=t_p}^{t-1} \eta_{\tau}^2 \mathbb{E}\left[\left\|\frac{1}{m} \sum_{i=1}^m \nabla F_i\left(x_i^{\tau}\right)\right\|^2\right] 
\notag \\
 & +\sigma_{G}^2(K-1)\left(A_1+1\right) \sum_{\tau=t_p}^{t-1} \eta_{\tau}^2 +{(K-1) \sigma_w^2} \sum_{\tau = t_p}^{t-1} \eta_{\tau}^2 \, ,
\end{align}
where the last inequality follows Assumption 3.
\end{proof}

\subsection{Proof of Theorem 3.5}

Under Assumptions 1-3, given $\left\{x_i^t\right\}_{i \in[m]}$, it follows from Lemmas ~\ref{lemma 1}-\ref{lemma 4} that

\begin{align}
&\mathbb{E}\left[\frac{1}{m} \sum_{i=1}^m F_i\left(\hat{x}_i^{t+1}\right)\right]  -\frac{1}{m} \sum_{i=1}^m F_i\left(\hat{x}_i^t\right) 
\leq \notag \\
&-\frac{\eta}{2}\left\|\frac{1}{m} \sum_{i=1}^m F_i\left(\hat{x}_i^t\right)\right\|^2-\frac{\eta}{2}\left\|\frac{1}{m} \sum_{i=1}^m F_i\left(x_i^t\right)\right\|^2+\frac{\eta L^2}{2} V^t \notag \\
 & +\frac{1}{2} \eta^2 L \lambda\left(\frac{A_1}{m}+A_2+1\right)\left\|\frac{1}{m} \sum_{i=1}^m F_i\left(x_i^t\right)\right\|^2 \notag \\
 & +\frac{1}{2} \eta^2 L \lambda\left\{\left(\frac{A_1}{m}+A_2+1\right) \sigma_{G}^2+\frac{\sigma_w^2}{m }+\sigma_{\beta}^2\right\} \, ,
\end{align}
where the expectation is taken with respect to the randomness in $\xi^t$. Thus, taking the unconditional expectation on both sides of the equation above, we have
\begin{align}
& \mathbb{E}\left[\frac{1}{m} \sum_{i=1}^m F_i\left(\hat{x}_i^{t+1}\right)-\frac{1}{m} \sum_{i=1}^m F_i\left(\hat{x}_i^t\right)\right] \notag \\
& \leq-\frac{\eta}{2} \mathbb{E}\left[\left\|\frac{1}{m} \sum_{i=1}^m F_i\left(\hat{x}_i^t\right)\right\|^2\right]-\frac{\eta}{2} \mathbb{E}\left[\left\|\frac{1}{m} \sum_{i=1}^m F_i\left(x_i^t\right)\right\|^2\right] \notag \\
& +\frac{\eta L^2}{2} \mathbb{E}\left[V^t\right] \\
&+ \frac{1}{2} \eta^2 L \lambda\left(\frac{A_1}{m}+A_2+1\right) \mathbb{E}\left[\left\|\frac{1}{m} \sum_{i=1}^m F_i\left(x_i^t\right)\right\|^2\right] \notag \\
&+ \frac{1}{2} \eta^2 L \lambda\left\{\left(\frac{A_1}{m}+A_2+1\right) \sigma_{G}^2+\frac{\sigma_w^2}{m }+\sigma_{\beta}^2\right\} \, ,
\end{align}

which implies that

\begin{align}
\label{Eq 38}
& \mathbb{E}\left[\frac{1}{m} \sum_{i=1}^m F_i\left(\hat{x}_i^{t_{p+1}}\right)-\frac{1}{m} \sum_{i=1}^m F_i\left(\hat{x}_i^{t_p}\right)\right] \notag \\
 =& \sum_{\tau=t_p}^{v_p} \mathbb{E}\left[\frac{1}{m} \sum_{i=1}^m F_i\left(\hat{x}_m^{\tau+1}\right)-\frac{1}{m} \sum_{i=1}^m F_i\left(\hat{x}_i^\tau \right)\right] \\
 \leq&-\frac{\eta}{2} \sum_{\tau=t_p}^{v_p} \mathbb{E}\left[\left\|\frac{1}{m} \sum_{i=1}^m F_i\left(\hat{x}_i^\tau\right)\right\|^2\right] \notag \\
+&\frac{\eta}{2}\left\{-1+\eta L \lambda\left(\frac{A_1}{m}+A_2+1\right)\right\}  \notag \\
&\sum_{\tau=t_p}^{v_p}\mathbb{E}\left[\left\|\frac{1}{m} \sum_{i=1}^m F_i\left(x_i^{\tau}\right)\right\|^2\right] \\
+&\frac{\eta L^2}{2} \sum_{\tau=t_p}^{v_p}\mathbb{E}\left[V^{\tau}\right] \notag \\
+&\sum_{\tau=t_p}^{v_p} \frac{1}{2} \eta^2 L \lambda\left\{\left(\frac{A_1}{m}+A_2+1\right) \sigma_{G}^2+\frac{\sigma_w^2}{m}+\sigma_{\beta}^2\right\} \, .
\end{align}

By Lemma~\ref{lemma 4}, for all $t_p \leq \tau \leq v_p$, we have that
\begin{align}
&\mathbb{E}\left[V^{\tau}\right] \\
 \leq & \lambda \eta^2(K-1)\left(A_1+1\right) \sum_{\tau=t_p}^{t-1} \mathbb{E}\left[\left\|\frac{1}{m} \sum_{i=1}^m \nabla F_i\left(x_i^{\tau}\right)\right\|^2\right] \notag \\
 +&\eta^2 \sigma_{G}^2(K-1)\left(A_1+1\right)\left(\tau-t_p\right)+{\eta^2 \sigma_w^2(K-1)}\left(\tau-t_p\right) \\
 \leq & \lambda \eta^2(K-1)\left(A_1+1\right) \sum_{\tau=t_p}^{v_p} \mathbb{E}\left[\left\|\frac{1}{m} \sum_{i=1}^m \nabla F_i\left(x_i^{\tau}\right)\right\|^2\right] \notag \\
+&\eta^2 \sigma_{G}^2(K-1)^2\left(A_1+1\right)+{\eta^2 \sigma_w^2(K-1)^2} .
\end{align}

Therefore, we have

\begin{align}
& \frac{\eta L^2}{2} \sum_{\tau=t_p}^{v_p} \mathbb{E}\left[V^{\tau}\right] \notag \\
\leq &\frac{1}{2} \lambda \eta^3 L^2(K-1) K\left(A_1+1\right) \sum_{\tau=t_p}^{v_p} \mathbb{E}\left[\left\|\frac{1}{m} \sum_{i=1}^m \nabla F_i\left(x_i^{\tau}\right)\right\|^2\right] \notag \\
+&\sum_{\tau=t_p}^{v_p}  \frac{1}{2} \eta^3 L^2 \sigma_{G}^2(K-1)^2\left(A_1+1\right) +\sum_{\tau=t_p}^{v_p} \frac{\eta^3 L^2 \sigma_w^2(K-1)^2}{2 }  .
\end{align}

Combined with \eqref{Eq 38}, we have

\begin{align}
&\mathbb{E} {\left[\frac{1}{m} \sum_{i=1}^m \nabla F_i\left(\hat{x}_i^{t_{p+1}}\right)-\frac{1}{m} \sum_{i=1}^m \nabla F_i\left(\hat{x}_i^{t_p}\right)\right] } 
\notag \\
\leq & -\frac{\eta}{2} \sum_{\tau=t_p}^{v_p} \mathbb{E}\left[\left\|\frac{1}{m} \sum_{i=1}^m \nabla F_i\left(\hat{x}_i^{\tau}\right)\right\|^2\right] \notag \\
 +& \frac{\eta}{2}\left\{-1+\eta L \lambda\left(\frac{A_1}{m}+A_2+1\right)+ \right. \notag \\
 &\left. \lambda \eta^2 L^2(K-1) K\left(A_1+1\right)\right\}\sum_{\tau=t_p}^{v_p} \mathbb{E}\left[\left\|\frac{1}{m} \sum_{i=1}^m \nabla F_i\left(x_i^{\tau}\right)\right\|^2\right] \notag \\
  +&\sum_{\tau=t_p}^{v_p} \frac{1}{2} \eta^2 L \lambda\left\{\left(\frac{A_1}{m}+A_2+1\right) \sigma_{G}^2+\frac{\sigma_w^2}{m}+\sigma_{\beta}^2\right\} \notag \\
&+\sum_{\tau=t_p}^{v_p}  \frac{1}{2} \eta^3 L^2 \sigma_{G}^2(K-1)^2\left(A_1+1\right) \notag \\
+&\sum_{\tau=t_p}^{v_p} \frac{\eta^3 L^2 \sigma_w^2(K-1)^2}{2} \, .
\end{align}

Due to
\begin{align}
-1+\eta L \lambda\left(\frac{A_1}{m}+A_2+1\right)+&\lambda \eta^2 L^2(K-1) K\left(A_1+1\right) \notag \\
\leq &0 \, ,
\end{align}
which implies that

\begin{align}
& \mathbb{E}\left[\frac{1}{m} \sum_{i=1}^m \nabla F_i\left(x_i^{t_{p+1}}\right)-\frac{1}{m} \sum_{i=1}^m \nabla F_i\left(x_i^{t_p}\right)\right] \notag \\
 \leq&-\frac{\eta}{2} \sum_{\tau=t_p}^{v_p} \mathbb{E}\left[\left\|\frac{1}{m} \sum_{i=1}^m \nabla F_i\left(\hat{x}_i^{\tau}\right)\right\|^2\right] \notag \\
+&\sum_{\tau=t_p}^{v_p} \frac{1}{2} \eta^2 L \lambda\left\{\left(\frac{A_1}{m}+A_2+1\right) \sigma_{G}^2+\frac{\sigma_w^2}{m}+\sigma_{\beta}^2\right\} \notag \\
+&\sum_{\tau=t_p}^{v_p} \frac{1}{2} \eta^3 L^2 \sigma_{G}^2(K-1)^2\left(A_1+1\right) \notag \\
+&\sum_{\tau=t_p}^{v_p} \frac{\eta^3 L^2 \sigma_w^2(K-1)^2}{2} .
\end{align}

Since we have assumed that $T=t_{\hat{p}}$ for some $\bar{p} \in \mathbb{N}^{+}$, we further have

\begin{align}
 & \frac{1}{T} \mathbb{E}\left[\left(\frac{1}{m} \sum_{i=1}^m \nabla F_i\left(\hat{x}_i^T\right)-f^*\right) \right.\notag \\
 &\left. -\left(\frac{1}{m} \sum_{i=1}^m \nabla F_i\left(\hat{x}_i^0\right)-f^*\right)\right]  \\
 =& \frac{1}{T} \mathbb{E}\left[\frac{1}{m} \sum_{i=1}^m \nabla F_i\left(\hat{x}_i^t\right)-\frac{1}{m} \sum_{i=1}^m \nabla F_i\left(\hat{x}_m^0\right)\right] \\
= & \frac{1}{T} \sum_{p=0}^{\bar{p}-1} \mathbb{E}\left[\frac{1}{m} \sum_{i=1}^m \nabla F_i\left(\hat{x}_m^{t_{p+1}}\right)-\frac{1}{m} \sum_{i=1}^m \nabla F_i\left(\hat{x}_m^{t_p}\right)\right] \\
\leq & -\frac{\eta}{2 K} \sum_{p=0}^{\bar{p}-1} \sum_{\tau=t_p}^{v_p} \mathbb{E}\left[\left\|\frac{1}{m} \sum_{i=1}^m \nabla F_i\left(\hat{x}_i^t\right)\right\|^2\right]  \notag \\
+ &
 \frac{1}{T} \sum_{p=0}^{\bar{p}-1} \sum_{\tau=t_p}^{v_p} 
\frac{1}{2} \eta^2 L \lambda\left\{\left(\frac{A_1}{m}+A_2+1\right) \sigma_{G}^2+\frac{\sigma_w^2}{m}+\sigma_{\beta}^2\right\} \\
 +&\frac{1}{T} \sum_{p=0}^{\bar{p}-1} \sum_{\tau=t_p}^{v_p} \frac{1}{2} \eta^3 L^2 \sigma_{G}^2(K-1)^2\left(A_1+1\right) \notag \\
+&\frac{1}{T} \sum_{p=0}^{\bar{p}-1} \sum_{\tau=t_p}^{v_p} \frac{\eta^3 L^2 \sigma_w^2(\tau-1)^2}{2 }  \\
= & -\frac{\eta}{2 T} \sum_{t=0}^{T-1} \mathbb{E}\left[\left\|\frac{1}{m} \sum_{i=1}^m \nabla F_i\left(\hat{x}_i^t\right)\right\|^2\right] \notag \\
+&\frac{1}{2} \eta^2 L \lambda\left\{\left(\frac{A_1}{m}+A_2+1\right) \sigma_{G}^2+\frac{\sigma_w^2}{m}+\sigma_{\beta}^2\right\} \\
 +&\frac{1}{2} \eta^3 L^2 \sigma_{G}^2(K-1)^2\left(A_1+1\right)+\frac{\eta^3 L^2 \sigma_w^2(K-1)^2}{2 } .
\end{align}

We can conclude that

\begin{align}
 &\frac{1}{T} \sum_{t=0}^{T-1} \mathbb{E}\left[\left\|\frac{1}{m} \sum_{i=1}^m \nabla F_i\left(\hat{x}_i^t\right)\right\|^2\right]  \notag \\
 \leq & \frac{2 \mathbb{E}\left[\frac{1}{m} \sum_{i=1}^m \nabla F_i\left(\hat{x}_i^0\right)-f^*\right]}{\eta T}
 \notag \\
 +&
 \eta L \lambda\left\{\left(\frac{A_1}{m}+A_2+1\right) \sigma_{G}^2 +\frac{\sigma_w^2}{m}+\sigma_{\beta}^2\right\} \notag \\
+&\eta^2 L^2 \sigma_{G}^2(K-1)^2\left(A_1+1\right)+{\eta^2 L^2 \sigma_w^2(K-1)^2}\, .
\end{align}

\section{Additional Experiment Results and Experiment Details}
\label{Sec app exp}

\subsection{Toy examples}
\label{sec app toy example}
In this subsection, we present the comprehensive comparison results illustrated in Figure 2 of the main paper, showcasing various levels of client cooperation, as depicted in Figure~\ref{complete compare}. The overall test accuracy performance is further demonstrated in Figure~\ref{test ACC on fashion}.

\begin{figure*}[t]
    \centering
    \includegraphics[width = 1.\textwidth]{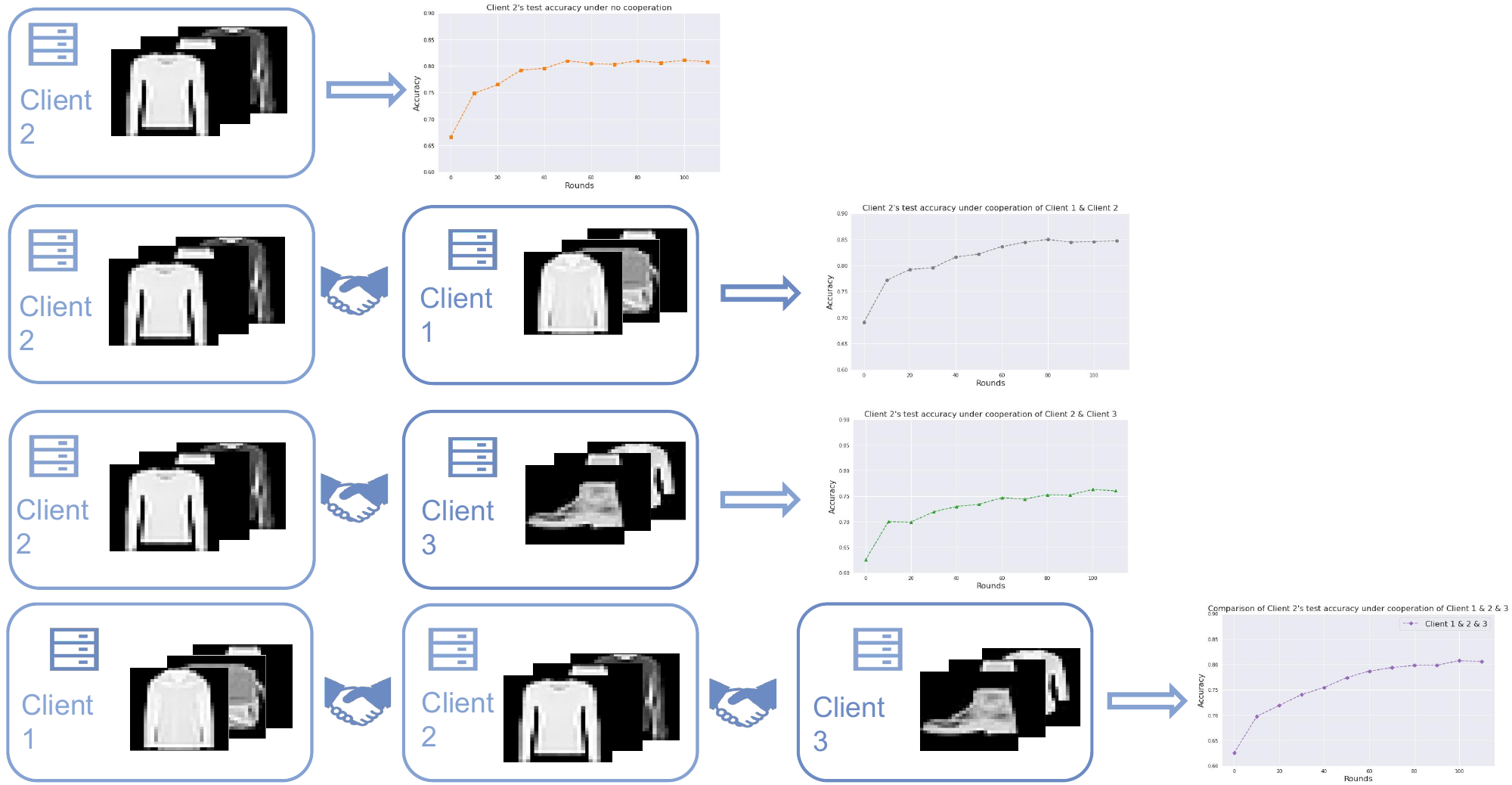}
    \caption{Illustration of different cooperation methods.}
    \label{complete compare}
\end{figure*}

\begin{figure}[t]
    \centering
    \includegraphics[width = .5 \textwidth]{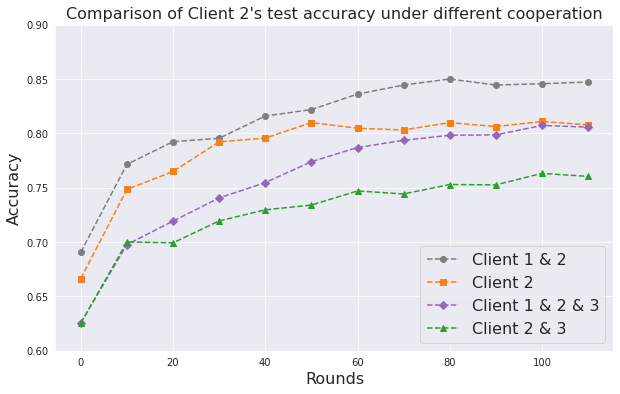}
    \caption{The test accuracy of different cooperation of clients on FashionMNIST.}
    \label{test ACC on fashion}
\end{figure}

\subsection{Experimental Environment}
\label{sec exp env}
For our experiments, we use NVIDIA GeForce RTX 3090 GPUs. Each simulation trail with 500 communication rounds and three random seeds.

\subsection{Experiment setup}
\label{app experiment setup}

\paragraph{Training Settings. } 
We maintain the same experimental settings for all baselines, conducting 500 communication rounds with 100 clients. The client sampling ratio is set to 0.1 for the baselines, while in DFL, when implementing sampling methods like gossip, each client communicates with 10 neighbors. The batch size is 64 for FEMNIST and 128 for CIFAR. The local epochs per round are set to 5 for all baselines, while the local fine-tune epoch is set to 1 by default. We employ SGD with momentum as the base optimizer with a learning rate of $\eta=0.01$ and a local momentum of $0.9$.

\paragraph{Setup for FashionMNIST, CIFAR-10 and CIFAR-100.}
To evaluate the performance of our algorithm AFIND+, we train a two-layer CNN on the non-iid FashionMNIST and four-layer CNN on CIFAR-10 datasets, and a ResNet-18 on the non-iid CIFAR-100 dataset, respectively.

Unless specifically mentioned otherwise, our studies use the following protocol: the default non-IID is from Dirichlet with a parameter of $\alpha = 0.5$, the server chooses $10\%$ clients according to sampling strategy from the total of $m=100$ clients, and each is trained for $T=500$ communication rounds with $K=5$ local epochs. 

All sampling algorithms use FedAvg-FT as the backbone. We compare our AFIND+ with Gossip sampling, PENS, and FedeRiCo on different datasets and different settings.

\paragraph{Setup for LEAF.} To test our algorithm's efficiency on diverse real datasets, we use the non-IID FEMNIST dataset in LEAF, as given in \citep{caldas2018leaf}. All baselines use a 4-layer CNN with a learning rate of $0.1$, batch size of 32, sample ratio of $10\%$ and communication round of $T=500$. The reported results are averaged over three runs with different random seeds.

\subsection{Additional Experimental Results}
\label{app additional experiments}

\paragraph{Ablation study for $\tau$ of AFINE+ on CIFAR-10.}
Figure~\ref{ablation tau} shows the convergence performance of AFIND+ on the CIFAR-10 dataset with $\alpha=0.1$, under different global thresholds, specifically $\tau = 0.1, 0.5, 0.8, 1.0, 1.5$. Different thresholds lead to different client participation behaviors for DFL, demonstrating the stability of our algorithm to the hyperparameter $\tau$. Additionally, a smaller $\tau=0.2$ tends to result in better performance, but the requirement of training rounds for convergence is larger. In contrast, a larger $\tau$ leads to faster convergence, but the accuracy is slightly poorer since the participation number of clients is small when the threshold is large.

\begin{figure}[t]
    \centering
    \includegraphics[width = .7 \textwidth]{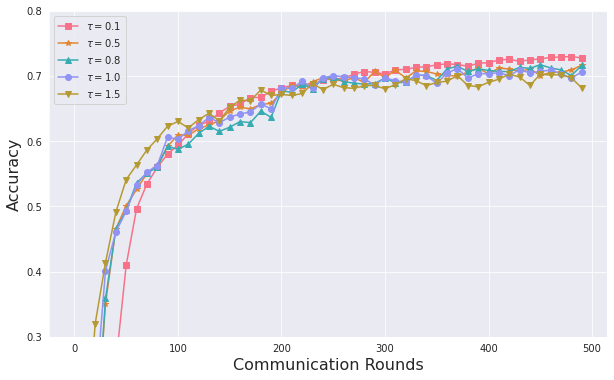}
    \caption{Ablation study for $\tau$ on CIFAR-10.}
    \label{ablation tau}
    \vspace{-1 em}
\end{figure}

\paragraph{Ablation study for network topologies.} 
Table~\ref{table of ablation of network typology} shows the performance of different sampling methods when applied to a partially connected network topology, in which each client randomly connects to half of the other clients. As the network typology becomes more challenging (from full connection to partial connection), the advantage of AFIND+ becomes more significant than other baselines because AFINE+ can adaptively adjust each client's sampling numbers, enabling flexible neighbor connections.

\begin{table}[!t]
 \small
 \centering
 \caption{Performance of sampling algorithms on the \textbf{partially connected network}. In the partially connected network, each client is randomly connected to others. In the full connection network typology (Table 1), each client is connected to all others.} 
 \fontsize{7.6}{11}\selectfont 
 \resizebox{.7\textwidth}{!}{%
  \begin{tabular}{l l l l l l l l l c c}
   \toprule
   \multirow{2}{*}{Algorithm} & \multicolumn{2}{c}{CIFAR-10} & \multicolumn{2}{c}{CIFAR-100}\\
   \cmidrule(lr){2-3} \cmidrule(lr){4-5} 
                    & $\alpha=0.1$ & $c=5$     &  $\alpha=0.1$ & $c=10$\\
   \midrule
Gossip  &51.35{\transparent{0.5}±1.76} & 68.12{\transparent{0.5}±2.47} & 49.61{\transparent{0.5}±1.74}  &58.64{\transparent{0.5}±1.38} \\ 
PENS &53.78{\transparent{0.5}±1.32} & 70.46{\transparent{0.5}±1.57} & 51.24{\transparent{0.5}±1.53}  &61.01{\transparent{0.5}±1.94} \\ 
FedeRiCo &55.85{\transparent{0.5}±1.15} & 71.58{\transparent{0.5}±1.93} & 53.21{\transparent{0.5}±1.64}  &63.41{\transparent{0.5}±1.57} \\ 
AFIND+ &65.24{\transparent{0.5}±2.12} & 76.21{\transparent{0.5}±2.12} & 58.87{\transparent{0.5}±2.18}  &66.85{\transparent{0.5}±2.15} \\ 
\bottomrule
\end{tabular}
}
\centering
\label{table of ablation of network typology}
\end{table}

\begin{table}[!t]
 \small
 \centering
 \caption{Performance of AFIND+ under \textbf{differential privacy noise}. Insert Gaussian
noise into the intermediate regularization variable $\delta$ with noise standard deviation $\sigma_2$. } 
 \fontsize{7.6}{11}\selectfont 
 \resizebox{.7\textwidth}{!}{%
  \begin{tabular}{c l l l l l l l l c c}
   \toprule
   \multirow{2}{*}{noise $\sigma_2$} & \multicolumn{2}{c}{CIFAR-10} & \multicolumn{2}{c}{CIFAR-100}\\
   \cmidrule(lr){2-3} \cmidrule(lr){4-5} 
                    & $\alpha=0.1$ & $c=5$     &  $\alpha=0.1$ & $c=10$\\
   \midrule
0 &71.89{\transparent{0.5}±0.06} & 80.36{\transparent{0.5}±0.32} & 62.11{\transparent{0.5}±1.82}  &71.07{\transparent{0.5}±1.02} \\ 
2 &70.61{\transparent{0.5}±0.14} & 79.67{\transparent{0.5}±1.34} & 61.07{\transparent{0.5}±0.40}  &70.82{\transparent{0.5}±1.51} \\ 
5 &69.94{\transparent{0.5}±1.42} & 78.85{\transparent{0.5}±0.62} & 60.72{\transparent{0.5}±0.16}  &68.23{\transparent{0.5}±1.23} \\ 
10 &67.07{\transparent{0.5}±0.42} & 76.11{\transparent{0.5}±1.74} & 58.41{\transparent{0.5}±1.52}  &66.05{\transparent{0.5}±0.87} \\ 
50 &64.24{\transparent{0.5}±1.56} & 73.92{\transparent{0.5}±1.08} & 55.86{\transparent{0.5}±1.12}  &63.20{\transparent{0.5}±1.18} \\ 
\bottomrule
\end{tabular}
}
\centering
\label{table of ablation of dp}
\end{table}

\paragraph{Privacy evaluation.}
We also evaluate AFIND+ under privacy preservation. Following~\cite{abadi2016deep},  we insert Gaussian
noise into the intermediate regularization variable $\delta$ with noise standard deviation $\sigma_2: \tilde{\sigma}_i \leftarrow \sigma_i + \frac{1}{L}\mathcal{N}(0, \sigma_2^2C_0^2 I)$, where $L$ is the batch size, $\sigma_2$ is the noise parameter, $C_2$ is the clipping constant. The result is shown in Table~\ref{table of ablation of dp}. With $\sigma_2 \leq 5$, AFIND+ shows only marginal reductions without significant performance degradation. However, higher values of $\sigma_2$ risk compromising performance. This suggests that our approach is compatible with a specific threshold of privacy preservation.

% \begin{table}[!t]
%  \small
%  \centering
%  \fontsize{7.6}{11}\selectfont 
%  \resizebox{0.7\textwidth}{!}{%
%   \begin{tabular}{l l l l l l l l l c c}
%    \toprule
%    \multirow{2}{*}{Algorithm} & \multicolumn{2}{c}{CIFAR-10} & \multicolumn{2}{c}{CIFAR-100}\\
%    \cmidrule(lr){2-3} \cmidrule(lr){4-5} 
%                     & $\alpha=0.1$ & $c=5$     &  $\alpha=0.1$ & $c=10$\\
%    \midrule
% Gossip  &51.35{\transparent{0.5}±1.76} & 68.12{\transparent{0.5}±2.47} & 49.61{\transparent{0.5}±1.74}  &58.64{\transparent{0.5}±1.38} \\ 
% PENS &53.78{\transparent{0.5}±1.32} & 70.46{\transparent{0.5}±1.57} & 51.24{\transparent{0.5}±1.53}  &61.01{\transparent{0.5}±1.94} \\ 
% FedeRiCo &55.85{\transparent{0.5}±1.15} & 71.58{\transparent{0.5}±1.93} & 53.21{\transparent{0.5}±1.64}  &63.41{\transparent{0.5}±1.57} \\ 
% AFIND+ &65.24{\transparent{0.5}±2.12} & 76.21{\transparent{0.5}±2.12} & 58.87{\transparent{0.5}±2.18}  &66.85{\transparent{0.5}±2.15} \\ 
% \bottomrule
% \end{tabular}
% }
% \centering
% \caption{\small {\textbf{Performance of sampling algorithms on the partially connected network.} In the partially connected network, each client is randomly connected to others, while in the full connection network typology (Table~\ref{table of performance comparison}), each client is connected to all others.} 
% \label{table of ablation of network typology}
% \end{table}

\end{document}